\documentclass[final]{l4dc2024}
\usepackage{hyperref}
\usepackage{graphicx,lipsum,wrapfig}%
\makeatletter
\let\Ginclude@graphics\@org@Ginclude@graphics 
\makeatother
\title[RF Approximation for Control]{Random Features Approximation for Control-Affine Systems
}

\usepackage{times}                           
\usepackage{verbatim}
\usepackage{algorithm,algcompatible}
\usepackage{amsfonts}
\usepackage{bm}
\usepackage{float}
\usepackage{setspace}

\newcommand{\R}{\mathbb{R}}

\newcommand{\norm}[1]{\left\lVert #1\right\rVert}

\newcommand{\diag}{\mathrm{diag}}

\newcommand{\derp}[2]{\frac{\partial #1 }{\partial #2 }}
\newcommand{\distas}{\overset{\text{i.i.d.}}{\sim}}
\newcommand{\bv}[1]{{\bm{#1}}}

\newcommand{\vSigma}{\bvgrk{\Sigma}}
\newcommand{\li}{\left<}
\newcommand{\ri}{\right>}

\newcommand{\bvgrk}[1]{{\boldsymbol{#1}}}	
\newcommand{\vphi}{\bvgrk{\phi}}
\newcommand{\vpsi}{\bvgrk{\psi}}
\newcommand{\vvarphi}{\bvgrk{\varphi}}
\newcommand{\vomega}{\bvgrk{\omega}}

\newcommand{\vPsi}{\bvgrk{\Psi}}
\newcommand{\vXi}{\bvgrk{\Xi}}
\newcommand{\vOmega}{\bvgrk{\Omega}}
\newcommand{\vvartheta}{\bvgrk{\vartheta}}

\newcommand{\tr}[1]{#1^\top}

\newcommand{\vf}{\bv{f}}
\newcommand{\vg}{\bv{g}}

\newcommand{\vk}{\bv{k}}

\newcommand{\vq}{\bv{q}}

\newcommand{\vs}{\bv{s}}

\newcommand{\vu}{\bv{u}}
\newcommand{\vv}{\bv{v}}
\newcommand{\vw}{\bv{w}}
\newcommand{\vx}{\bv{x}}
\newcommand{\vy}{\bv{y}}
\newcommand{\vz}{\bv{z}}
\algnewcommand\INPUT{\item[\textbf{Input:}]}%
\algnewcommand\OUTPUT{\item[\textbf{Output:}]}%
\newcommand{\vA}{\bv{A}}
\newcommand{\vB}{\bv{B}}
\newcommand{\vC}{\bv{C}}
\newcommand{\vD}{\bv{D}}
\newcommand{\vE}{\bv{E}}

\newcommand{\vG}{\bv{G}}

\newcommand{\vI}{\bv{I}}

\newcommand{\vK}{\bv{K}}

\newcommand{\vM}{\bv{M}}

\newcommand{\vU}{\bv{U}}

\author{\Name{Kimia Kazemian} \Email{kk983@cornell.edu}\\\Name{Yahya Sattar} \Email{ysattar@cornell.edu}\\\Name{Sarah Dean} \Email{sdean@cornell.edu}\\
\addr Department of Computer Science, Cornell University, Ithaca, NY, USA.
 }

\begin{document}
\maketitle

\begin{abstract}
	Modern data-driven control applications call for flexible nonlinear models that are amenable to principled controller synthesis and realtime feedback.
	Many nonlinear dynamical systems of interest are \emph{control affine}.
	We propose two novel classes of nonlinear feature representations which capture control affine structure while allowing for arbitrary complexity in the state dependence.
	Our methods make use of random features (RF) approximations,
	inheriting the expressiveness of kernel methods 
	at a lower computational cost.
	We formalize the representational capabilities of our methods by showing their relationship to the Affine Dot Product (ADP) kernel proposed by~\citet{contro1} and a novel Affine Dense (AD) kernel that we introduce.
	We further illustrate the utility by presenting a case study of data-driven optimization-based control using control certificate functions (CCF).
	Simulation experiments on a double pendulum empirically demonstrate the advantages of our methods.
\end{abstract}

\section{Introduction} 
Modern control applications require modelling systems with complex and nonlinear dynamics.
Modern machine learning techniques offer a data-driven solution.
From deep learning to kernel methods, 
learning-based approaches fit models to data.
Highly expressive models can approximate arbitrary functions, and therefore model arbitrarily complex phenomena.
However, this comes at a cost---they can be computationally expensive to train
and difficult to use for the purpose of synthesizing a controller.
This poses a challenge in real-time feedback systems.

Linear regression is a straightforward approach for learning dynamical models from data, so long as a suitable nonlinear feature representation, i.e., set of basis functions, is known~\citep{mania2020active}.
However, selecting proper basis functions is often challenging and requires modelling detailed properties of the unknown dynamics. 
One solution is to choose a set of random basis functions to generate feature vectors of fixed dimension. 
This approach, called \emph{random features~(RF)}, can achieve high expressiveness as long as the dimension of the feature vectors is large enough~\citep{randombasis}. 
Random features have proven useful for dynamical systems forecasting~\citep{giannakis2023learning}, receding horizon control~\citep{lale2021model}, and policy learning~\citep{lale2022kcrl}.

We propose two novel classes of random feature representations suitable for principled data-driven control (Section~\ref{sec:RF_for_CAM}). 
Our key insight is to leverage the control-affine structure of many nonlinear dynamics of interest,
which enables principled optimization-based approaches for controller synthesis.
We propose two distinct methods for incorporating this structure into random basis functions and
formalize their representation guarantees by showing that they approximate functions in a Reproducing Kernel Hilbert Space (RKHS).  
One of our methods approximates the \emph{Affine Dot Product~(ADP) kernel} proposed by~\citet{contro1}, while the other corresponds to a novel \emph{Affine Dense (AD) kernel} that we propose.
The RF methods significantly reduce the computational time and memory complexity
compared to their kernel counterparts.

To showcase the utility of an explicit control-affine structure, we present a case study for nonlinear control in Section~\ref{sec:case_study_CFC}.
Our data-driven approach is based on \emph{Control Certificate Functions} (CCFs), which are utilized to synthesize controllers that provably achieve properties such as safety and stability~\citep{ccf}. 
CCFs have been used in a range of applications from robotics to multi-agent systems \citep{1,2,4,6},
including in a data-driven manner~\citep{contro1,adp,ccf,choi2023constraint}.
Simulations on a double inverted pendulum illustrate the benefits of our models when used in a \emph{certainty-equivalent} manner.
In Appendix \ref{affinenessinu}, we additionally
derive uncertainty estimates analogous to those of Gaussian process (GP) regression, and use them to propose a robust data-driven controller in Appendix \ref{sec:robust_ccf}.
We highlight that the approximation methods that we propose may be broadly of interest for any control application which makes use of GPs~\citep{ellipsoid, Learningbased_nonlinearMPC, NMPC,cautiousMPC,stochasticMPC}.


\section{Problem Setting and Preliminaries}

In this work, we consider an affine modelling and prediction problem inspired by applications in data-driven control. 
We first define the general problem of interest, and then give several examples that arise in the context of learning for dynamics and control.

\begin{definition}[Control-affine modelling problem]\label{def:mainprob}
	For data of the form $\{(\vx_i, \vu_i, z_i)\}_{i=1}^N$, find a function $\hat h\!:\!\mathcal X\!\times \!\mathcal U\!\to\!\mathbb R$ which i) is affine in its second argument and ii) accurately models the relationship between $(\vx, \vu)$ and $z$, i.e. $\hat h(\vx_i,\vu_i)$ is not far from $z_i$.
\end{definition}

Such a modelling problem naturally arises in applications involving nonlinear control-affine systems.
The dynamics are described in either continuous or discrete time: 
\begin{align}
	\dot{\vx} = \vf(\vx) + \vg(\vx)\vu\quad\text{or}\quad \vx_{t+1} = \vf(\vx_t) + \vg(\vx_t)\vu_t.
	\label{eq:dynamics}
\end{align}
Here, $\vx \! \in \!\mathcal{X} \!\subseteq \!\mathbb{R}^n$ is the system state, and $\vu\! \in\!\mathcal U \!\subseteq \!\mathbb{R}^m$ is the control input.
The nonlinear function $\vf$ determines the evolution of the state in the absence of control inputs, while $\vg$ models state-dependent actuation.
Control-affine dynamics arise naturally from manipulator equations~\citep{murray1991case,underactuated}, and are thus prevalent in applications like robotics.
While many systems are known to follow dynamics of the form~\eqref{eq:dynamics}, the precise form of $\vf$ and/or $\vg$ may be unknown.
Data-driven approaches enable the control of systems with entirely or partially unknown dynamics.
There are many examples of modelling tasks that arise in such data-driven control settings.

\begin{example}\label{example_01}
	Consider a model predictive control setting in which the evolution of the state itself must be predicted~\citep{lale2021model}.
	For a discrete-time control-affine system~\eqref{eq:dynamics} with unknown dynamics and direct state observation,    
	a sequence of states and inputs $\{(\vx_i, \vu_i)\}_{i=0}^N$ defines $n$ modelling problems of the form presented in Definition~\ref{def:mainprob}: one for each state dimension. 
\end{example}

\begin{example}\label{example_02}
	Consider again model predictive control, now for continuous time control-affine dynamics~\eqref{eq:dynamics}.    
	A sequence of sampled states $\{\vx_i\}_{i=0}^N$ can be used to approximate $\{\dot{\vx}_i\}_{i=1}^N$  with forward finite differencing and define $n$ modelling problems of the form presented in Definition~\ref{def:mainprob}.
\end{example}

\begin{example}\label{example_03}
	Consider Certificate Function Control~\citep{ccf}, which enforces safety or stability using a known certificate function $C\!:\!\mathcal X\!\to \!\mathbb R$.
	For continuous time control-affine dynamics~\eqref{eq:dynamics},
	such controllers require computing  $\dot{C}\!:\!\mathcal X \!\times\! \mathcal U\!\to \!\R$, which cannot be done directly when the dynamics are unknown. However, a sequence $\{C(\vx_i)\}_{i=0}^N$ can be computed from a sequence of sampled states and the known function $C$.
	Then, finite differencing approximates the time derivative, resulting in a problem of the form presented in Definition~\ref{def:mainprob}. 
\end{example}

\begin{example}\label{example_04}
	For any of the previous examples, suppose that an approximate model of the dynamics $\tilde{\vf}$ and $\tilde{\vg}$ is known.
	Then learning residual error dynamics also results in a problem of the form presented in Definition~\ref{def:mainprob} (see e.g., \citet{ccf, adp}).
\end{example}

The examples above serve to motivate the relevance of the modelling problem in Definition~\ref{def:mainprob}.
We now turn to background and preliminaries on solving it.
Our focus is on nonparametric techniques which can model phenomena of arbitrary complexity.
We review kernel regression, which is both nonparametric and amenable to uncertainty quantification, and random features approximation, which allows for computational efficiency.


\subsection{From Linear to Kernel Regression}\label{subsec:Bayesian Regression}
We begin by reviewing regression approaches for general data containing input vectors $\{\vs_i\}_{i=1}^N \!\subset\!\R^d$ and a target output variable $\{z_i\}_{i=1}^N\!\subset\!\R$.
Our starting point is parametric linear regression, 
in which predictions depend linearly on a known nonlinear feature function of the inputs.
Let $\vphi\!:\!\R^d\!\to\!\R^D$ map an input vector $\vs\! \in \!\R^d$ to a feature vector $\vphi(\vs)\! \in \!\R^D$.
The feature function, also known as a basis function, maps the input vectors to a higher-dimensional feature space, where a linear relationship can be established more easily. 

Linear least-squares regression \citep{watson1967linear} models the relationship as $\hat h(\vs) \!= \!\hat{\vw}^\top \vphi(\vs)$, where the parameter $\hat{\vw}\!\in\!\mathbb R^D$ is learned from data by solving
\begin{align}
	\min_{\vw \in\mathbb R^D } \sum_{i=1}^N (\vphi(\vs_i)^\top \vw - z_i)^2 + \lambda \|\vw\|_2^2, 
\end{align}
where $\lambda\geq 0$ is a regularization parameter.
Let the matrix $\Phi\!\in\!\R^{N\times D}$ and the vector $\vz \!\in \!\R^N$ be the aggregation of rows $\{\vphi(\vs_i)^\top\}_{i=1}^N$ and $\{z_i\}_{i=1}^N$, respectively. 
Then the prediction is
\begin{align}
	& \hat h(\vs) = \vphi(\vs)^\top(\Phi^\top\Phi+\lambda \vI_D)^{-1}\Phi^\top\vz = \vphi(\vs)^\top\Phi^\top(\Phi\Phi^\top+\lambda \vI_N)^{-1}\vz.
	\label{prediction} 
\end{align}

The first equality is the closed-form solution to the least squares objective, and the second leverages the \emph{kernel trick}~\citep{scholkopf2018learning,muller2018introduction}.
The significance of this reformulation is that the vector $\Phi\vphi(\vs)=: \vk(\vs)$ and matrix $\Phi\Phi^\top =: K$  can be computed using only inner products of basis functions evaluated on training data.
This is attractive because the class of basis functions determine the complexity and richness of the modelled relationship between inputs $\vs$ and outputs $z$. 
While low-dimensional bases may suffice for highly structured processes, generally, a suitable compact basis may not be known a priori.

Kernel methods allow for expressive basis functions of arbitrarily high or infinite dimension.
A kernel function $k\!:\!\R^d\!\times \!\R^d\!\to \!\R$  generalizes inner products between basis functions, and is used as a nonparametric approach for representing complex functions.
Appropriately defined, kernel ridge regression corresponds to regression in Reproducing Kernel Hilbert Spaces (RKHS)~\citep{wendland2004scattered}. 
Many RKHS are dense in the set of continuous functions, enabling arbitrarily accurate representation of continuous functions via kernel regression.
The following lemma presents a sufficient condition for checking that a kernel function defines a RKHS.
It is a direct implication of the Moore–Aronszajn theorem and Lemma 1 in \citet{berlinet2011reproducing}.
\begin{lemma}\label{lem:rkhs}
	Let $\mathcal H$ be some Hilbert space with inner product $\langle \cdot,\cdot\rangle$.
	A function $k\!:\!\R^d\!\times \!\R^d\!\to \!\R$ is a reproducing kernel if there exists a mapping $\varphi\!:\!\R^d\! \to \!\mathcal H$ such that $k(\vs,\vs')\!=\!\langle\varphi(\vs),\varphi(\vs')\rangle$.
\end{lemma}

In addition to their expressivity, kernel methods are amenable to theoretical guarantees and uncertainty characterization. 
Popularized from the Bayesian perspective as Gaussian Process (GP) regression~\citep{gpbook},
confidence intervals on kernel predictions can be derived even in frequentist settings~\citep{GPtheorem}. 
We discuss this perspective further in appendix \ref{sec:gp_background} as it is useful for robust control.
The drawback of kernel methods is computation. Algorithms have superlinear complexity in the number of data points. In particular, computing the kernel weights can be prohibitively expensive for large datasets. Solving \eqref{prediction} generally requires $O(N^3)$ time and $O(N^2)$ memory.

\subsection{Random Feature Approximation}\label{sec:rfbg}

Rather than using kernel methods directly, we propose basis functions which are expressive, general purpose, and yet finite-dimensional. 
Consider a parametric family of basis functions $b\!:\!\R^d\!\times \!\R^p\!\to \!\R$.
Then for parameters $\{\vvartheta_j\}_{j=1}^D \!\subset \!\R^p$ sampled i.i.d. from a fixed probability distribution $p(\vvartheta)$, the random basis is defined as 
$\vphi(\vs)=\begin{bmatrix}b(\vs;\vvartheta_1) & b(\vs;\vvartheta_2)& \dots &  b(\vs;\vvartheta_D)\end{bmatrix}^\top$.
Random basis functions of this form approximate rich class of functions in the sense that $\vphi(\vs)^\top\vphi(\vs')$ is a Monte-Carlo estimator which converges uniformly to a kernel $k(\vs,\vs')$~\citep{randombasis}. 
The rate of convergence is controlled by the feature dimension $D$ and the particular kernel depends on the definition of $b(\cdot;\vvartheta)$ and $p(\vvartheta)$.

The underlying observation behind random features is a simple consequence of Bochner’s Theorem~\citep{avron2017random}: For every normalized shift-invariant kernel (i.e., $k(0) = 1$), there is 
a probability density function $p(\cdot)$ on $\R^d$ such that
\begin{align}\label{bochner}
	k(\vs,\vs') 
	=  \int_{ \vvartheta \in \R^d} e^{-i 2 \pi  \vvartheta^\top (\vs - \vs')}  p(\vvartheta) d \vvartheta =: \mathcal{F}(p(\vvartheta)).
\end{align}
In other words, the inverse Fourier transform $\mathcal{F}^{-1}$ of the kernel $k(\cdot)$ is the probability density function $p(\cdot)$. This implies a one-to-one correspondence between any shift invariant kernel and a random features basis.
\begin{example}[Random Fourier basis]\label{ex:rff}
	The random Fourier basis consists of sinusoidal nonlinearities of the form 
	$b(\vs;\vvartheta)\!=\!\begin{bmatrix} \cos(\vvartheta^\top \vs ) & \sin(\vvartheta^\top \vs)\end{bmatrix}.$
	When $\vvartheta$ is sampled from a Gaussian distribution, i.e., $p(\vvartheta) \!\sim \!\mathcal{N}(0,2\gamma \vI_d)$,
	then the random Fourier basis approximates the radial basis function~(RBF) kernel $k(\vs,\vs')\!= e^{-\frac{1}{\gamma}\|\vs-\vs'\|_2^2}$.\end{example}

The randomized nonlinear expansions provide a compact and computationally efficient alternative to the RKHS representations.
This is particularly attractive when the number of data points is large.
Recall from the prior section the feature matrix $\Phi\! \in \!\R^{N \!\times D}$ appearing in the prediction~\eqref{prediction}. 
Since $\vphi(\vs)\!\in \!\R^D$ and $\Phi^\top\Phi$ is a $D\!\times\!D$ matrix, the computation only depends on the dimension of our feature space. 
Hence, we can compute a random feature approximation in $O(ND^2)$ time and $O(ND)$ memory, which is computationally attractive when $D\!<\!N$. 


\section{Random Features for Control-Affine Modelling}\label{sec:RF_for_CAM}
In this section, we use ideas from kernel regression and random feature approximations to propose representations which capture the control-affine structure from Definition~\ref{def:mainprob}. 
We first present two general approaches for defining basis functions that are affine in the control variable. 
Then, we present RKHS representation guarantees by showing that the random basis approximates particular kernels.
Finally, we present experiments which illustrate the predictive modelling capabilities of the proposed methods.

\subsection{Control-Affine Basis Functions}\label{sec:comparison}
The control-affine modelling problem (Definition~\ref{def:mainprob}) allows for complex dependence on the state variable, but imposes a restriction on the control variable.
Given any arbitrary state-dependent bases $\vpsi_i\!: \!\mathcal{X} \!\to\! \R^D$ for $i=1,...,m+1$, we propose the following two basis functions that are affine in the control variable $\vu$. 

\begin{definition}[Affine dot product (ADP) bases]\label{def:adpb} 
	The basis $\vphi_c\!:\!{\mathcal X}\!\times\! \mathcal U\!\rightarrow\! \R^{D(m+1)}$, given by \[\vphi_c(\vx,\vu)=\begin{bmatrix} u_1\vpsi_1(\vx)^\top &\dots&  u_{m}\vpsi_{m} (\vx)^\top &\vpsi_{m+1}(\vx)^\top\end{bmatrix}^\top,\]
	is the ADP basis of $m\!+\!1$ individual basis functions $\vpsi_i\!:\!{\mathcal X}\!\rightarrow \!\R^D$, $i=1,\dots,m\!+\!1$. 
\end{definition}

\textcolor{black}{As we show in the following section~(see Theorem~\ref{thrm:adp_approx}), the ADP bases approximate the affine dot product (ADP) kernel, which was first proposed by~\citet{contro1}.} 
Note that the ADP bases can also
be written as the product of  $\mathrm{blkdiag}(\vpsi_1(\vx),...,\vpsi_{m+1}(\vx))$ with the vector $[\vu^\top~~1]^\top$. 
This basis expands the feature dimension for every dimension of the control input, resulting in dimension which scales by $m\!+\!1$.
This observation motivates a second proposed representation.

\begin{definition}[Affine dense (AD) bases]\label{def:adb}
	The basis
	$\vphi_d\!:\!\mathcal X\!\times\!\mathcal U\!\rightarrow\! \R^D$, given by 
	\[\vphi_d(\vx,\vu)=\begin{bmatrix}\vpsi_1(\vx) \dots \vpsi_{m\!+\!1}(\vx)\end{bmatrix}\begin{bmatrix}\vu \\ 1\end{bmatrix}\]
	is the AD basis of $m+1$ individual basis functions $\vpsi_i\!:\!{\mathcal X}\!\rightarrow \!\R^D$, $i=1,\dots m\!+\!1$. 
\end{definition}

Compared with the ADP basis, the AD basis is more compact.
For individual basis functions of dimension $D$, the AD basis will be of dimension $D$, whereas the ADP basis will be of dimension $D(m\!+\!1)$.
Considering the linear regression use case,
this means that AD has $O(ND^2)$ time and $O(ND)$ memory complexity, whereas for ADP it is $O(N(m+1)^2D^2)$ and $O(N(m+1)D)$.

Leveraging ideas from random Fourier features, we propose control-affine basis functions constructed with state-dependent random Fourier basis functions $\vpsi_i$ for $i=1,\dots,m\!+\!1$: 
\begin{align}\label{eq:rfx}
	\vpsi_i(\vx):= \sqrt{2/D}\begin{bmatrix}\sin(\vvartheta_{i,1}^\top \vx) &
		\cos(\vvartheta_{i,1}^\top \vx) & \dots
		&\sin(\vvartheta_{i,D/2}^\top \vx) & \cos(\vvartheta_{i,D/2}^\top \vx)\end{bmatrix}^\top,
\end{align}
where weights $\{\vvartheta_{i,j}\}_{j=1}^{D/2}$ are drawn i.i.d. from the distribution $p_i(\vvartheta)$ for $i\!=\!1,\dots,m\!+\!1$. 
As described in Example~\ref{ex:rff},
each of these individual basis functions approximates a shift invariant kernel corresponding to the Fourier transform of the density $p_i(\vvartheta)$~\citep{randombasis,rahimi}.
In other words,
$\mathbb E_{\vvartheta\sim p_i(\cdot)}[\vpsi_i(\vx)^\top\vpsi_i(\vx')]=k_i(\vx-\vx')$ where $k_i(\vv)=\mathcal F(p_i(\vvartheta))[\vv]$.


\subsection{Representation Guarantees} \label{subsec:rep_guarantee}

We now develop representation guarantees for the compound bases by showing which kernels they approximate.
The affine dot product (ADP) kernel was first proposed by~\citet{contro1} for systems with control-affine dynamics.
\begin{definition}[Affine dot product (ADP) kernel] \label{def:ADP_kernel}
	Define $k_{c}\!:\!{\mathcal X}\!\times\!\mathcal \!U \!\times \!{\mathcal X}\!\times\!\mathcal U \!\rightarrow \!\mathbb{R}$, given by
		\[k_{c}((\vx,\vu), (\vx',\vu')) := \begin{bmatrix}\vu^\top & 1\end{bmatrix} \diag(k_1(\vx, \vx'), \cdots, k_{m+1}(\vx, \vx')) \begin{bmatrix}\vu'^\top & 1\end{bmatrix}^\top, \]
	as the Affine Dot Product (ADP) compound kernel of $m\!+\!1$ individual kernels $k_i\!:\!\mathcal{X}\!\times\! \mathcal{X} \!\rightarrow \!\R$.
\end{definition}

The following theorem shows that the ADP basis approximates the ADP kernel.
\begin{theorem}[ADP Approximation] \label{thrm:adp_approx}
	For $i=1,...,m\!+\!1$, suppose the basis functions $\vpsi_i$ are defined according to~\eqref{eq:rfx} with $p_i$ the inverse Fourier transform of a shift invariant kernel $k_i$. 
	Let $\vphi_c$ be the ADP compound basis of $\vpsi_i$ and let $k_c$ the compound ADP kernel of $k_i$.
	Then $$\mathbb E[\vphi_c(\vx,\vu)^\top \vphi_c(\vx',\vu')]=k_{c}((\vx,\vu), (\vx',\vu')).$$
\end{theorem}
The result follows by relating the dot product of features to the diagonal matrix in the ADP kernel.
We defer all formal proofs to Appendix \ref{sec:proofs}.

An alternative way to understand the ADP random feature approximation is to interpret the $(m\!+\!1)\!\times \!(m\!+\!1)$ diagonal matrix of kernels as an operator valued kernel.
This operator valued kernel is the sum of $m\!
+\!1$ \emph{decomposable kernels}, as defined by \citet{brault2016random} (Definition 3).
The ADP block diagonal matrix of basis functions can be interpreted through their framework as a random feature approximation for this operator valued kernel.




We now turn to the affine dense basis. First, we define a novel Affine Dense compound kernel.

\begin{definition}[Affine dense (AD) kernel] \label{def:AD_kernel}
	For $m\!+\!1$ individual kernels $k_i\!:\!\mathcal{X}\!\times \!\mathcal{X} \!\rightarrow \!\R$, 
	let $\vD(\vx,\vx')$ be the diagonal matrix with $i^{th}$ entry as $k_i(\vx-\vx')$, and $\vA(\vx,\vx')$ a matrix with zero on the diagonal and $[\vA(\vx,\vx')]_{ij}=k_i(\vx)k_j(\vx')$  for $i\neq j\in[m+1]$. 
	Then, define the Affine Dense (AD) compound kernel as
	$k_{d}:{\mathcal X}\!\times\!\mathcal U \!\times \!{\mathcal X}\!\times\!\mathcal \!U \!\rightarrow \!\mathbb{R}$, given by
		\[k_{d}((\vx,\vu), (\vx',\vu')) := \begin{bmatrix}\vu^\top & 1\end{bmatrix} \left( \vD(\vx,\vx')+\vA(\vx,\vx') \right)
		\begin{bmatrix}\vu'^\top & 1\end{bmatrix}^\top. \]
\end{definition}

Notice that the diagonal matrix $\vD(\vx,\vx')$ in the AD kernel is similar to the ADP kernel.
However, the AD kernel additionally includes the dense matrix $\vA(\vx,\vx')$.
Due to this second dense term, the AD compound kernel is not shift invariant in $\vx$.
As a result, it is not possible to view $\vA\!+\!\vD$ as a shift invariant operator-valued kernel, and thus the results of \cite{brault2016random} cannot be used to derive a random features approximation.
Furthermore, it is not immediately clear whether the AD kernel is indeed a valid reproducing kernel.
We therefore begin by showing that it is.

\begin{theorem}[AD kernel]\label{thm:ad-kernel}
	Let $k_d((\vx,\vu), (\vx',\vu'))$ be as in Definition~\ref{def:AD_kernel}. Suppose each $k_i(\vx,\vx')$ is a normalized shift invariant reproducing kernel. Then, $k_d((\vx,\vu), (\vx',\vu'))$ is a reproducing kernel.
\end{theorem}

To prove this result, we use the crucial (but non-obvious) claim that
if $k(\vx,\vx')$ is a normalized shift invariant reproducing kernel, then $k(\vx,\vx')\!-\!k(\vx)k(\vx')$ is also a reproducing kernel.
To prove that the claim is true, we construct an explicit feature mapping of the form required in Lemma~\ref{lem:rkhs}.
With the claim in hand, the proof follows by algebraic manipulations and the fact that the set of reproducing kernels is closed under addition.
Therefore, the AD kernel $k_d$ is a reproducing kernel and thus defines a RKHS.
We next show that the AD basis functions approximate this RKHS.

\begin{theorem}[Affine-dense kernel approximation] \label{thrm:Affine-dense rf}
	Suppose that for $i=1,...,m\!+\!1$ the basis functions $\vpsi_i$ are defined according to~\eqref{eq:rfx} with $p_i$ the inverse Fourier transform of a shift invariant kernel $k_i$. 
	Let $\vphi_d$ be the AD compound basis of $\vpsi_i$ and let $k_d$ be the compound AD kernel of $k_i$.
	Then
	$\mathbb E[\vphi_d(\vx,\vu)^\top \vphi_d(\vx',\vu')]=k_{d}((\vx,\vu), (\vx',\vu'))$.
\end{theorem}

So far our results show that the basis functions we propose approximate kernel regression in expectation. 
When the dimension $D$ is large enough, the approximation error can be bounded with high probability~\citep{rahimi,analysis}.
In particular, \cite{analysis} show conditions under which the pointwise approximation error  is
no more than $\epsilon$ with probability depending on $D$ and $\epsilon$.
We therefore conclude with a result which shows that when the individual kernels have bounded approximation errors, so do the compound kernels.
In Appendix \ref{sec:errorbounds}, we further derive bounds on the prediction errors and confidence intervals for use in robust control.
\begin{proposition}[Kernel Approximation Errors]\label{prop:kernelapprox}
	Consider the ADP kernel $k_c(\vs,\vs')$ and the AD kernel $k_d(\vs,\vs')$ from Definitions~\ref{def:ADP_kernel} and \ref{def:AD_kernel}, respectively. Consider $\{k_i(\vx)\}_{i=1}^{m\!+\!1}$ which are the individual kernels used to construct the ADP and the AD kernels. Recall $\vpsi$ from \ref{eq:rfx}. Suppose $|k_i(\vx)| \!\leq\! 1$ and $|k_i(\vx) - \vpsi_i(\vx)^\top\vpsi_i(\vx)| \!\leq \!\epsilon$ for all $ \vx \in \mathcal{X}$ and $i \!\in\! [m+1]$. Then, we have 
	\begin{align}
		\max\{|k_c(\vs,\vs')-\vphi_c(\vs)^\top \vphi_c(\vs')|,~ |k_d(\vs,\vs')-\vphi_d(\vs)^\top \vphi_d(\vs')|\} & \leq \epsilon(\vu^\top \vu'+1)
	\end{align}
	where $\vphi_c$ and $\vphi_d$ are defined in Theorems \ref{thrm:adp_approx} and \ref{thrm:Affine-dense rf} respectively.
\end{proposition}



\begin{figure}[t!]
	\begin{centering}
		\includegraphics[width=0.48\linewidth]{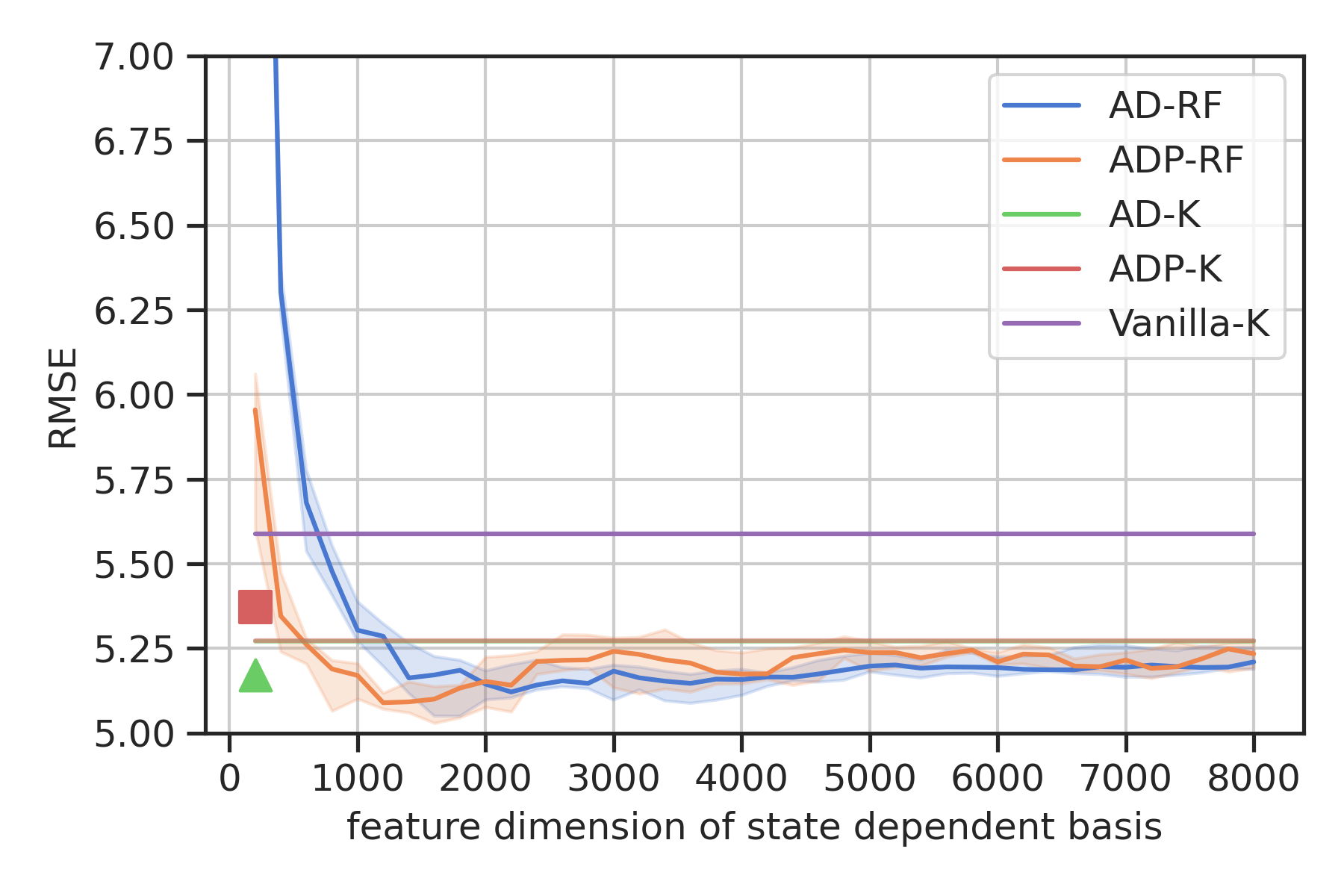} 
		~~~
		\includegraphics[width=0.48\linewidth]{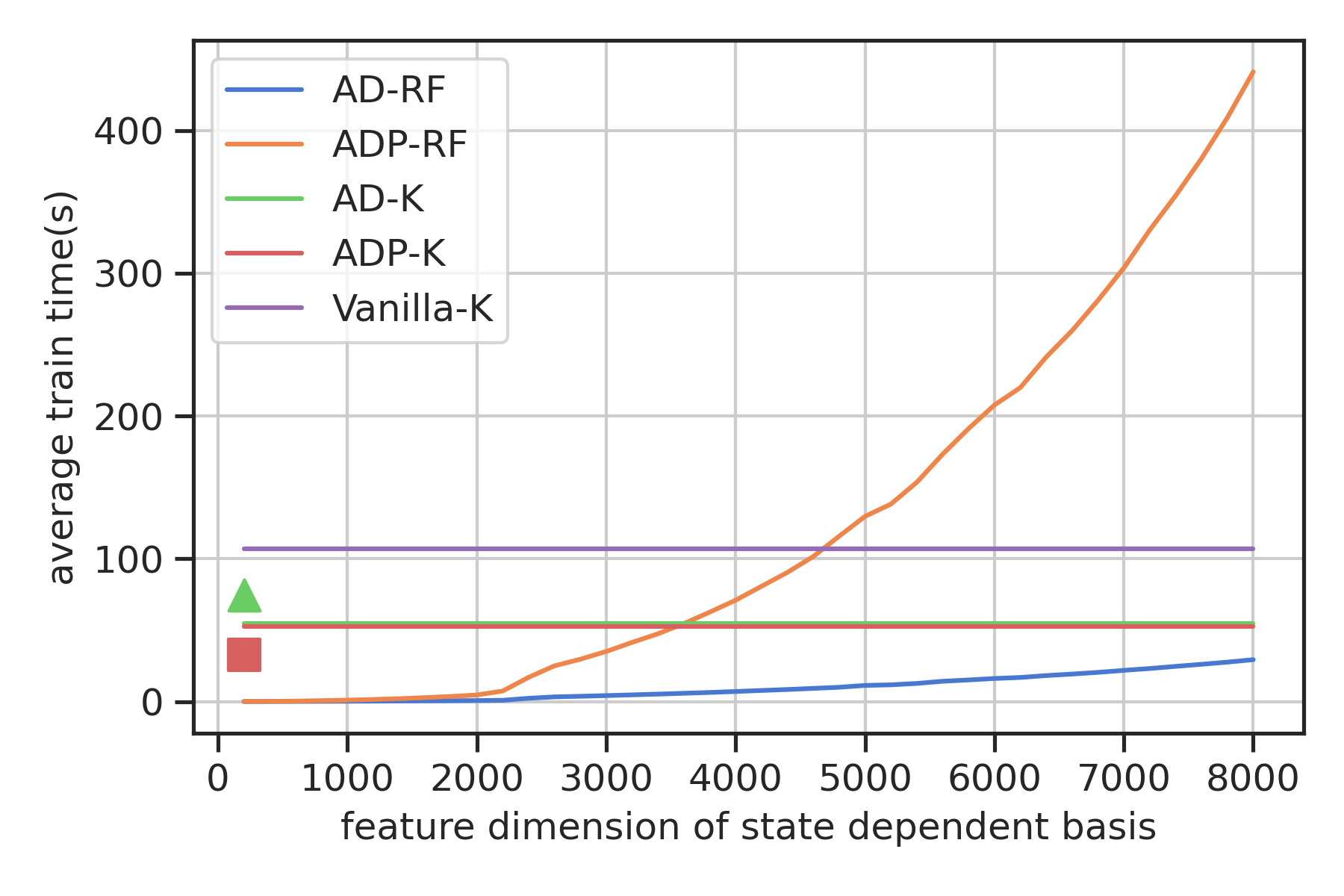}
	\end{centering}
	\caption{Evaluation of models, comparing prediction accuracy on test data (left) and training time on 8859 points (right) for a prediction problem on a double pendulum system. 
		Horizontal lines and markers correspond to kernel methods.
		Random features are sampled 10 times at varying dimensions; the left panel displays median and quartiles over the trials while the right panel shows the mean.
		Increasing the feature dimension of each state-dependent basis $\vpsi_i(x)$ results in lower RMSE but longer training time, especially for ADP-RF.}\vspace{-20pt}
	
	\label{fig:prediction}
\end{figure} 

\vspace{-5pt}
\subsection{Numerical demonstration}\label{sec:pred-exp}
In this section, we empirically\footnote{Code is available at \url{ https:github.com/kimzemian/swift_affine_mastery}} study the performance of the the two random features methods (ADP-RF and AD-RF) as well as the corresponding kernel methods (ADP-K and AD-K).
We focus on performance in terms of prediction accuracy.
In the next section, we also demonstrate the utility of these models for data-driven control.

We consider a prediction task relating to a double pendulum with actuation at both joints.
The state of this system $\vx\!\in\!\mathbb R^4$ consists of two angle variables and two angular velocities, while the control input $\vu\!\in\!\mathbb R^2$ consists of the two joint actuation torques.
In appendix \ref{sec:dipderivation}, we present a full derivation of the dynamics equation, which is affine in the control inputs.
We simulate the system under closed-loop control and sample at 10 Hz.
The controller is imperfectly designed to bring the system to an upright and balanced configuration; 
Further controller details are deferred to the following section.
We collect a dataset containing 226 trajectories, each starting at a different initial point and lasting 5 seconds.
The dataset if of the form 
$\{\{(\vx^e_i,\vu^e_i), z^e_i\}_{i=1}^L\}_{e=1}^E$
where $z_i$ is the time derivative of a scalar function of the state
(Example~\ref{example_03}); details are described in the following section.
We split the data into train and evaluation subsets with an $80/20$ split, so the train size is 8859 and test size is 2215, formulating a prediction task of the form in Definition~\ref{def:mainprob}.

We compare the performance of five models: three kernel methods (Vanilla-K, ADP-K, AD-K) and two random features methods (ADP-RF, AD-RF). 
Vanilla-K is an RBF kernel (Example~\ref{ex:rff}) that operates on the concatenated state and input without any affine structure.
ADP-K and AD-K (Definitions~\ref{def:ADP_kernel} and~\ref{def:AD_kernel}) use RBF kernel on the state variable, whereas, ADP-RF and AD-RF are as defined in Theorems~\ref{thrm:adp_approx} and~\ref{thrm:Affine-dense rf} with the corresponding random Fourier bases~\eqref{eq:rfx} as in Example~\ref{ex:rff}.  
For all models, $\gamma\!=\!1$ and $\lambda\!=\!1$.
Figure~\ref{fig:prediction} plots the performance in terms of test accuracy and training time.
The left panel shows median and quartile RMSE on the evaluation split
and the right panel shows the average training time.
The kernels are represented by horizontal lines and markers.
Vanilla-K performs worse than the affine kernels since it does not capture the affine structure, while AD-K and ADP-K have similar performance.
For the RF models, we examine the effect of the random features approximation of state-dependent samples $\vpsi_i(x)$ of dimension $D$.
ADP-RF has lower error for smaller feature dimension, but both RF methods quickly approach the performance of the kernel methods.
For small $D$, the RF methods are both much faster.
Train time increases with $D$ more quickly for ADP-RF than AD-RF, as training ADP-RF scales quadratically with $m\!+\!1$.
\textcolor{black}{Comparing the RF models, Figure~\ref{fig:prediction} suggests that, although training with AD-RF is faster as compared to ADP-RF, the later has smaller RMSE. 
	We attribute this to the higher dimensionality of the ADP compound basis, which allows for greater expressivity.
	In Appendix \ref{sec:rf_comparison}, we present extensive experiments with synthetic data demonstrating the relationship between training time and RMSE.
	These show that for fixed training time, AD-RF outperforms ADP-RF on accuracy when $D$ is sufficiently large, and this performance advantage grows as input dimension $m$ increases.}


\section{Case Study: Certificate Function Control}\label{sec:case_study_CFC}
\label{sec:ccf_control}

A key motivation for our work is that the affine structure of our data-driven models is amenable for use in control tasks.
We therefore describe how to incorporate these models into a particular approach to nonlinear control.
We then evaluate closed loop performance of our models.

\paragraph{Background} As a case study, we demonstrate a nonlinear control technique based on \emph{control certificate functions} as proposed by \cite{ccf}, which we refer to for a more rigorous and precise introduction. This approach generalizes and unifies the use of control Lyapunov functions (CLFs) to guarantee stability~\citep{galloway2015torque} and control barrier functions (CBFs) to guarantee safety~\citep{ames2016control}.
Certificate function control requires a continuously differentiable \emph{certificate function} $C:\mathbb{R}^n \to \mathbb{R}$ satisfying certain properties \citep{ccf}, along with a \emph{comparison function} $\alpha:\R\to\R_+$.
Then, an optimization-based state feedback controller
can be defined which will guarantee desired properties such as stability or safety by construction~\citep{ames2019control}.
Given some ``desired'' control input $\vu_d(\vx)$,
the CCF quadratic program (QP) is: \
\begin{align*}
	\tag{\textbf{CCF-QP}} \label{eq:CCF-QP}
	\vu^{*}(\vx) = ~&\underset{ \vu \in \mathbb{R}^m}{\arg\min} \norm{\vu}_2^2 + c_1\norm{\vu - \vu_d(\vx)}_2^2\\
	&\text{ s.t. }\,\underbrace{\nabla C(\vx)^\top (\vf(\vx) + \vg(\vx)\vu)}_{\dot C(\vx,\vu)}+ \alpha(C(\vx)) \leq 0.
\end{align*}



\paragraph{Data-driven Control}
We now suppose that the dynamics $\vf$ and $\vg$ are unknown, so the
CCF-QP controller cannot be directly implemented.
We assume that a valid CCF $C$ and comparison function $\alpha$ for the unknown true system is given\footnote{This assumption is met for feedback linearizable systems as long as the the degree of actuation of the true dynamics model is known~\citep{taylor2019episodic}. For example, many robotic systems satisfy this assumption.}.
Given a sampled trajectory $\{(\vx_i,\vu_i)\}_{i=1}^N$,
we construct a
control affine modelling problem for $\dot C\!:\!\mathcal X\!\times\!\mathcal U\!\to\!\R$ as described in Example~\ref{example_03}.
We therefore use methods discussed in Section~\ref{sec:RF_for_CAM} to create a control-affine model $\hat h(\vx, \vu)$
which can be used in place of the unknown $\dot C(\vx,\vu)$ function in~\eqref{eq:CCF-QP} in a \emph{certainty equivalent} (CE) manner.
Because the model $\hat h$ is affine in $\vu$, the resulting optimization problem is still a QP.
In Appendix \ref{sec:ccf_modeling},
we provide additional details on constructing this QP for both kernel and RF methods.
We also discuss methods for \emph{robust}, rather than CE, data-driven control. 
The robust approach requires estimates of 
uncertainty (e.g. as in Gaussian process regression) as well as the pointwise  RF approximation errors,
and results in a second order cone program (SOCP), see Appendix \ref{sec:robust_ccf}.




\paragraph{Simulation experiments}
We simulate data-driven CCF control of the double pendulum introduced in Section~\ref{sec:pred-exp},
where the goal is to swing up and balance in the upright position $\vx\!=\!0$ with only an incorrect model of the dynamics $\tilde{\vf},\tilde{\vg}$.
Knowing only its degree of actuation,
we may conclude that the dynamics are \emph{feedback linearizable} and therefore
we can define a Control Lyapunov Function (CLF) without the exact dynamics model~\citep{taylor2019episodic}.
Specifically, we define $C(\vx) =\vx^\top P \vx$, $\alpha(\vx)=.725\vx$,
$\vu_d(\vx)$ a feedback linearizing controller for the incorrect $\tilde{\vf},\tilde{\vg}$, and $c_1=25$. 
Full details are provided in Appendices \ref{sec:data_collection}, \ref{sec:closed_loop}.

We first define a ``nominal'' QP controller which selects inputs according to~\eqref{eq:CCF-QP} with the nominal model $\tilde{\vf},\tilde{\vg}$. 
We use this controller to gather trajectories and define a dataset as described in Section~\ref{sec:pred-exp}.
We subsample the data by $1/5$
and derive data-driven models of $\dot C(\vx,\vu)$ as outlined in the paragraph above.
We consider four data-driven QP controllers using the 
four affine models: AD-K, ADP-K, AD-RF, and ADP-RF. 
For the data-driven controllers, we augment the initial dataset with episodic data collection: we run the controller for 10 seconds at 10 Hz, retrain, and repeat for ten episodes.
The RF dimension is $D=N/5$ for $N$ the size of the training data.
Finally, we compare the performance of the the nominal and data-driven methods with an ``oracle'' controller that solves~\eqref{eq:CCF-QP} with the true dynamics. 
Figure~\ref{fig:closed_loop} plots the system trajectory in terms of the Lyapunov function $C(\vx)$ and the pendulum configuration.
While the nominal controller fails to balance the pendulum, the data-driven controllers succeed and are similar to each other.

\begin{figure}[t!]
	\centering
	\includegraphics[width=0.45\linewidth]{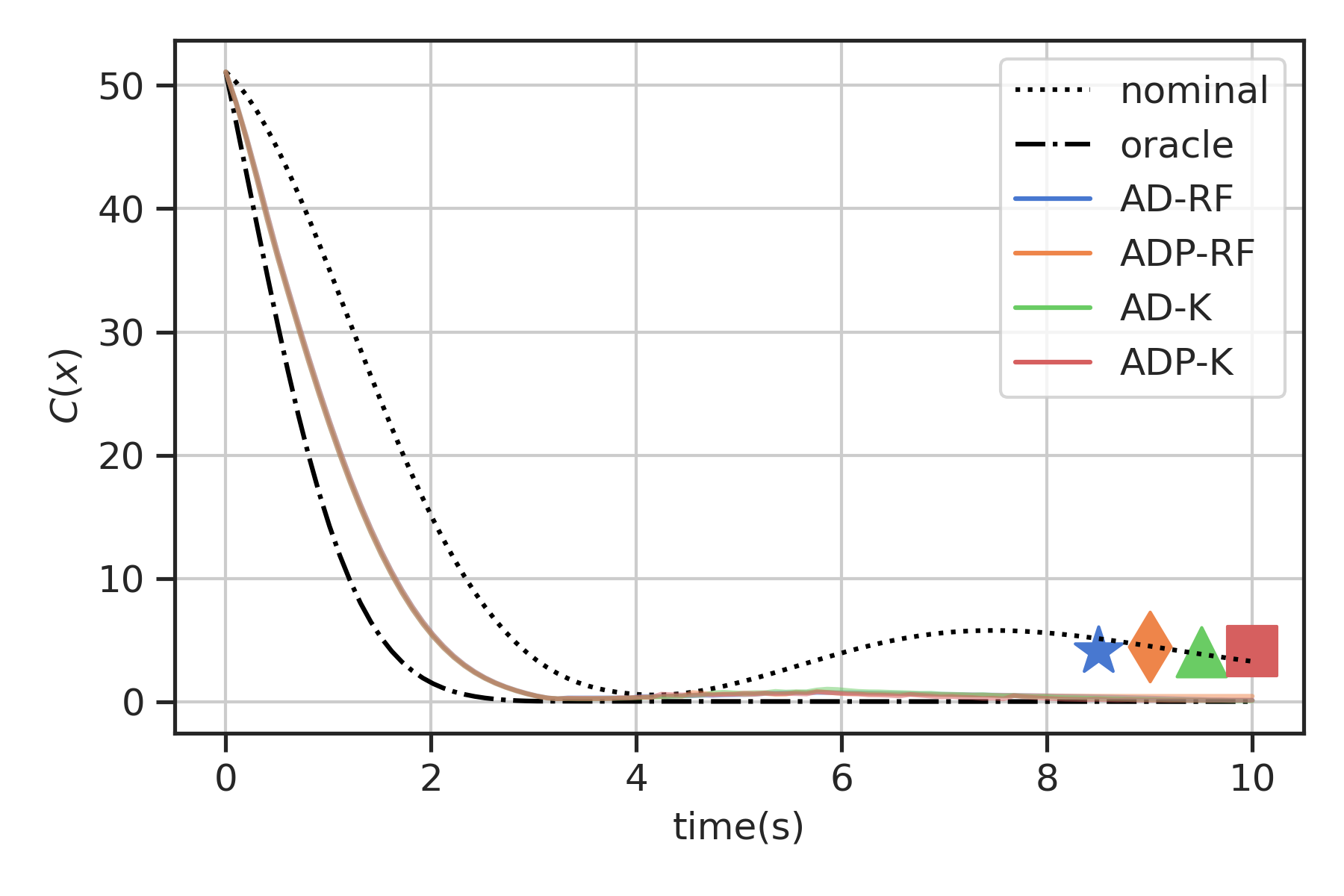}
	\includegraphics[width=0.52\linewidth]{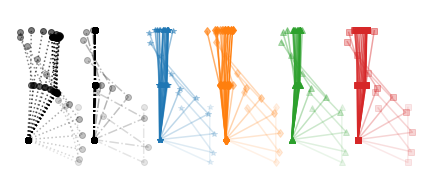}
	\caption{Left: The value of the Lyapunov function $C(\vx)$ over time for nominal, oracle, and data-driven controllers with initial state $[2,0,0,0]$. Right: Illustration of the pendulum configurations over time. Nominal fails to balance the pendulum; data-driven methods succeed.}
	\vspace{-20pt}
	\label{fig:closed_loop}
\end{figure}


\section{Conclusion}\label{sec:conc}
This work considers a control affine modelling problem and proposes two classes of random basis functions as a solution: ADP and AD.
The representation guarantees of these methods are made formal by connection to kernel regression in corresponding RKHSs.
A case study in nonlinear control with CCF illustrates the utility of control affine models.
Numerical experiments demonstrate the
performance of the RF and kernel methods in terms of accuracy, computation time, and closed loop control performance.
In Appendix \ref{sec:errorbounds}, we additionally present
uncertainty estimates analogous to Gaussian process (GP) regression, as well as a corresponding robust data-driven control.
We highlight that the approximation methods that we propose may be broadly of interest for any control application which makes use of GPs.

Our work opens the door to many future questions of interest.
It would be interesting to develop kernels and random features tailored to particular control applications.
One could explore the application of our methods to additional control techniques, like feedback linearization or model predictive control.
It would be interesting to develop principled techniques for acquiring data, expanding from the simple warm start episodic approach that we used.
Furthermore, additional methods for approximating kernels would provide alternatives to speeding up kernel and GP regression for data-driven control.
\acks{We thank Jason Choi for helping us understand their code base for ADP kernel. This work was partly funded by NSF CCF 2312774 and NSF OAC-2311521, a LinkedIn Research Award, and a
gift from Wayfair.}
\bibliography{citation}

\begin{thebibliography}{36}
\providecommand{\natexlab}[1]{#1}
\providecommand{\url}[1]{\texttt{#1}}
\expandafter\ifx\csname urlstyle\endcsname\relax
  \providecommand{\doi}[1]{doi: #1}\else
  \providecommand{\doi}{doi: \begingroup \urlstyle{rm}\Url}\fi

\bibitem[Ames et~al.(2014)Ames, Grizzle, and Tabuada]{2}
Aaron~D Ames, Jessy~W Grizzle, and Paulo Tabuada.
\newblock Control barrier function based quadratic programs with application to
  adaptive cruise control.
\newblock In \emph{53rd IEEE Conference on Decision and Control}, pages
  6271--6278. IEEE, 2014.

\bibitem[Ames et~al.(2016)Ames, Xu, Grizzle, and Tabuada]{ames2016control}
Aaron~D Ames, Xiangru Xu, Jessy~W Grizzle, and Paulo Tabuada.
\newblock Control barrier function based quadratic programs for safety critical
  systems.
\newblock \emph{IEEE Transactions on Automatic Control}, 62\penalty0
  (8):\penalty0 3861--3876, 2016.

\bibitem[Ames et~al.(2019)Ames, Coogan, Egerstedt, Notomista, Sreenath, and
  Tabuada]{ames2019control}
Aaron~D Ames, Samuel Coogan, Magnus Egerstedt, Gennaro Notomista, Koushil
  Sreenath, and Paulo Tabuada.
\newblock Control barrier functions: Theory and applications.
\newblock In \emph{2019 18th European control conference (ECC)}, pages
  3420--3431. IEEE, 2019.

\bibitem[Artstein(1983)]{1}
Zvi Artstein.
\newblock Stabilization with relaxed controls.
\newblock \emph{Nonlinear Analysis: Theory, Methods \& Applications},
  7\penalty0 (11):\penalty0 1163--1173, 1983.

\bibitem[Avron et~al.(2017)Avron, Kapralov, Musco, Musco, Velingker, and
  Zandieh]{avron2017random}
Haim Avron, Michael Kapralov, Cameron Musco, Christopher Musco, Ameya
  Velingker, and Amir Zandieh.
\newblock Random fourier features for kernel ridge regression: Approximation
  bounds and statistical guarantees.
\newblock In \emph{International conference on machine learning}, pages
  253--262. PMLR, 2017.

\bibitem[Berlinet and Thomas-Agnan(2011)]{berlinet2011reproducing}
Alain Berlinet and Christine Thomas-Agnan.
\newblock \emph{Reproducing kernel Hilbert spaces in probability and
  statistics}.
\newblock Springer Science \& Business Media, 2011.

\bibitem[Bradford et~al.(2019)Bradford, Imsland, and del Rio-Chanona]{NMPC}
Eric Bradford, Lars Imsland, and Ehecatl~Antonio del Rio-Chanona.
\newblock Nonlinear model predictive control with explicit back-offs for
  gaussian process state space models.
\newblock In \emph{2019 IEEE 58th Conference on Decision and Control (CDC)},
  pages 4747--4754, 2019.
\newblock \doi{10.1109/CDC40024.2019.9029443}.

\bibitem[Brault et~al.(2016)Brault, Heinonen, and Buc]{brault2016random}
Romain Brault, Markus Heinonen, and Florence Buc.
\newblock Random fourier features for operator-valued kernels.
\newblock In \emph{Asian Conference on Machine Learning}, pages 110--125. PMLR,
  2016.

\bibitem[Caldwell and Marshall(2021)]{Learningbased_nonlinearMPC}
Jack Caldwell and Joshua~A. Marshall.
\newblock Towards efficient learning-based model predictive control via
  feedback linearization and gaussian process regression.
\newblock In \emph{2021 IEEE/RSJ International Conference on Intelligent Robots
  and Systems (IROS)}, pages 4306--4311, 2021.
\newblock \doi{10.1109/IROS51168.2021.9636755}.

\bibitem[Casta{\~n}eda et~al.(2021)Casta{\~n}eda, Choi, Zhang, Tomlin, and
  Sreenath]{contro1}
Fernando Casta{\~n}eda, Jason~J Choi, Bike Zhang, Claire~J Tomlin, and Koushil
  Sreenath.
\newblock Gaussian process-based min-norm stabilizing controller for
  control-affine systems with uncertain input effects and dynamics.
\newblock In \emph{2021 American Control Conference (ACC)}, pages 3683--3690.
  IEEE, 2021.

\bibitem[Castañeda et~al.(2021)Castañeda, Choi, Zhang, Tomlin, and
  Sreenath]{adp}
Fernando Castañeda, Jason~J. Choi, Bike Zhang, Claire~J. Tomlin, and Koushil
  Sreenath.
\newblock Pointwise feasibility of gaussian process-based safety-critical
  control under model uncertainty.
\newblock In \emph{2021 60th IEEE Conference on Decision and Control (CDC)},
  pages 6762--6769, 2021.
\newblock \doi{10.1109/CDC45484.2021.9683743}.

\bibitem[Choi et~al.(2023)Choi, Casta{\~n}eda, Jung, Zhang, Tomlin, and
  Sreenath]{choi2023constraint}
Jason~J Choi, Fernando Casta{\~n}eda, Wonsuhk Jung, Bike Zhang, Claire~J
  Tomlin, and Koushil Sreenath.
\newblock Constraint-guided online data selection for scalable data-driven
  safety filters in uncertain robotic systems.
\newblock \emph{arXiv preprint arXiv:2311.13824}, 2023.

\bibitem[Galloway et~al.(2015)Galloway, Sreenath, Ames, and
  Grizzle]{galloway2015torque}
Kevin Galloway, Koushil Sreenath, Aaron~D Ames, and Jessy~W Grizzle.
\newblock Torque saturation in bipedal robotic walking through control lyapunov
  function-based quadratic programs.
\newblock \emph{IEEE Access}, 3:\penalty0 323--332, 2015.

\bibitem[Giannakis et~al.(2023)Giannakis, Henriksen, Tropp, and
  Ward]{giannakis2023learning}
Dimitrios Giannakis, Amelia Henriksen, Joel~A Tropp, and Rachel Ward.
\newblock Learning to forecast dynamical systems from streaming data.
\newblock \emph{SIAM Journal on Applied Dynamical Systems}, 22\penalty0
  (2):\penalty0 527--558, 2023.

\bibitem[Hewing et~al.(2020)Hewing, Kabzan, and Zeilinger]{cautiousMPC}
Lukas Hewing, Juraj Kabzan, and Melanie~N. Zeilinger.
\newblock Cautious model predictive control using gaussian process regression.
\newblock \emph{IEEE Transactions on Control Systems Technology}, 28\penalty0
  (6):\penalty0 2736--2743, 2020.
\newblock \doi{10.1109/TCST.2019.2949757}.

\bibitem[Koller et~al.(2018)Koller, Berkenkamp, Turchetta, and
  Krause]{ellipsoid}
Torsten Koller, Felix Berkenkamp, Matteo Turchetta, and Andreas Krause.
\newblock Learning-based model predictive control for safe exploration.
\newblock In \emph{2018 IEEE Conference on Decision and Control (CDC)}, pages
  6059--6066, 2018.
\newblock \doi{10.1109/CDC.2018.8619572}.

\bibitem[Lale et~al.(2021)Lale, Azizzadenesheli, Hassibi, and
  Anandkumar]{lale2021model}
Sahin Lale, Kamyar Azizzadenesheli, Babak Hassibi, and Anima Anandkumar.
\newblock Model learning predictive control in nonlinear dynamical systems.
\newblock In \emph{2021 60th IEEE Conference on Decision and Control (CDC)},
  pages 757--762. IEEE, 2021.

\bibitem[Lale et~al.(2022)Lale, Shi, Qu, Azizzadenesheli, Wierman, and
  Anandkumar]{lale2022kcrl}
Sahin Lale, Yuanyuan Shi, Guannan Qu, Kamyar Azizzadenesheli, Adam Wierman, and
  Anima Anandkumar.
\newblock Kcrl: Krasovskii-constrained reinforcement learning with guaranteed
  stability in nonlinear dynamical systems.
\newblock \emph{arXiv preprint arXiv:2206.01704}, 2022.

\bibitem[Li et~al.(2021)Li, Li, and He]{stochasticMPC}
Fei Li, Huiping Li, and Yuyao He.
\newblock Adaptive stochastic model predictive control of linear systems using
  gaussian process regression.
\newblock \emph{IET Control Theory \& Applications}, 15\penalty0 (5):\penalty0
  683--693, 2021.

\bibitem[Mania et~al.(2020)Mania, Jordan, and Recht]{mania2020active}
Horia Mania, Michael~I Jordan, and Benjamin Recht.
\newblock Active learning for nonlinear system identification with guarantees.
\newblock \emph{arXiv preprint arXiv:2006.10277}, 2020.

\bibitem[M{\"u}ller et~al.(2018)M{\"u}ller, Mika, Tsuda, and
  Sch{\"o}lkopf]{muller2018introduction}
Klaus-Robert M{\"u}ller, Sebastian Mika, Koji Tsuda, and Koji Sch{\"o}lkopf.
\newblock An introduction to kernel-based learning algorithms.
\newblock In \emph{Handbook of neural network signal processing}, pages 4--1.
  CRC Press, 2018.

\bibitem[Murray and Hauser(1991)]{murray1991case}
Richard~M Murray and John~Edmond Hauser.
\newblock \emph{A case study in approximate linearization: The acrobat
  example}.
\newblock Electronics Research Laboratory, College of Engineering, University
  of~…, 1991.

\bibitem[Nguyen et~al.(2016)Nguyen, Hereid, Grizzle, Ames, and Sreenath]{4}
Quan Nguyen, Ayonga Hereid, Jessy~W Grizzle, Aaron~D Ames, and Koushil
  Sreenath.
\newblock 3d dynamic walking on stepping stones with control barrier functions.
\newblock In \emph{2016 IEEE 55th Conference on Decision and Control (CDC)},
  pages 827--834. IEEE, 2016.

\bibitem[Pickem et~al.(2017)Pickem, Glotfelter, Wang, Mote, Ames, Feron, and
  Egerstedt]{6}
Daniel Pickem, Paul Glotfelter, Li~Wang, Mark Mote, Aaron Ames, Eric Feron, and
  Magnus Egerstedt.
\newblock The robotarium: A remotely accessible swarm robotics research
  testbed.
\newblock In \emph{2017 IEEE International Conference on Robotics and
  Automation (ICRA)}, pages 1699--1706. IEEE, 2017.

\bibitem[Rahimi and Recht(2007)]{rahimi}
Ali Rahimi and Benjamin Recht.
\newblock Random features for large-scale kernel machines.
\newblock In \emph{Proceedings of the 20th International Conference on Neural
  Information Processing Systems}, NIPS'07, page 1177–1184, Red Hook, NY,
  USA, 2007. Curran Associates Inc.
\newblock ISBN 9781605603520.

\bibitem[Rahimi and Recht(2008)]{randombasis}
Ali Rahimi and Benjamin Recht.
\newblock Uniform approximation of functions with random bases.
\newblock In \emph{2008 46th annual allerton conference on communication,
  control, and computing}, pages 555--561. IEEE, 2008.

\bibitem[Schechter(2001)]{functionalanalysis}
Martin Schechter.
\newblock \emph{Principles of functional analysis}.
\newblock Number~36. American Mathematical Soc., 2001.

\bibitem[Scholkopf and Smola(2018)]{scholkopf2018learning}
Bernhard Scholkopf and Alexander~J Smola.
\newblock \emph{Learning with kernels: support vector machines, regularization,
  optimization, and beyond}.
\newblock MIT press, 2018.

\bibitem[Srinivas et~al.(2009)Srinivas, Krause, Kakade, and Seeger]{GPtheorem}
Niranjan Srinivas, Andreas Krause, Sham~M Kakade, and Matthias Seeger.
\newblock Gaussian process optimization in the bandit setting: No regret and
  experimental design.
\newblock \emph{arXiv preprint arXiv:0912.3995}, 2009.

\bibitem[Sutherland and Schneider(2015)]{analysis}
Dougal Sutherland and Jeff Schneider.
\newblock On the error of random fourier features.
\newblock \emph{Uncertainty in Artificial Intelligence - Proceedings of the
  31st Conference, UAI 2015}, 06 2015.

\bibitem[Taylor et~al.(2019)Taylor, Dorobantu, Le, Yue, and
  Ames]{taylor2019episodic}
Andrew~J Taylor, Victor~D Dorobantu, Hoang~M Le, Yisong Yue, and Aaron~D Ames.
\newblock Episodic learning with control lyapunov functions for uncertain
  robotic systems.
\newblock In \emph{2019 IEEE/RSJ International Conference on Intelligent Robots
  and Systems (IROS)}, pages 6878--6884. IEEE, 2019.

\bibitem[Taylor et~al.(2021)Taylor, Dorobantu, Dean, Recht, Yue, and Ames]{ccf}
Andrew~J. Taylor, Victor~D. Dorobantu, Sarah Dean, Benjamin Recht, Yisong Yue,
  and Aaron~D. Ames.
\newblock Towards robust data-driven control synthesis for nonlinear systems
  with actuation uncertainty.
\newblock In \emph{2021 60th IEEE Conference on Decision and Control (CDC)},
  pages 6469--6476, 2021.
\newblock \doi{10.1109/CDC45484.2021.9683511}.

\bibitem[Tedrake(2023)]{underactuated}
Russ Tedrake.
\newblock \emph{Underactuated Robotics}.
\newblock 2023.
\newblock URL \url{https://underactuated.csail.mit.edu}.

\bibitem[Watson(1967)]{watson1967linear}
Geoffrey~S Watson.
\newblock Linear least squares regression.
\newblock \emph{The Annals of Mathematical Statistics}, pages 1679--1699, 1967.

\bibitem[Wendland(2004)]{wendland2004scattered}
Holger Wendland.
\newblock \emph{Scattered data approximation}, volume~17.
\newblock Cambridge university press, 2004.

\bibitem[Williams and Rasmussen(2006)]{gpbook}
Christopher~KI Williams and Carl~Edward Rasmussen.
\newblock \emph{Gaussian processes for machine learning}, volume~2.
\newblock MIT press Cambridge, MA, 2006.

\end{thebibliography}
\newpage
\appendix

\section{Omitted Proofs}\label{sec:proofs}
\subsection{Proof of Theorem~\ref{thrm:adp_approx}}
First note that for the ADP basis, \[\vphi_c(\vx,\vu)^\top \vphi_c(\vx',\vu') =  \begin{bmatrix}\vu^\top & 1\end{bmatrix} \diag(\vpsi_1(\vx)^\top \vpsi_1(\vx'), \cdots, \vpsi_{m+1}(\vx)^\top \vpsi_{m+1}(\vx')) \begin{bmatrix}\vu'^\top & 1\end{bmatrix}^\top.\]
The result holds because by assumption, the expectation of $\vpsi_i(\vx)^\top \vpsi_i(\vx')$ equals $k_i(\vx-\vx')$.

\subsection{Proof of Theorem~\ref{thm:ad-kernel}}
Consider the following claim:
if $k(\vx,\vx')$ is a normalized shift invariant reproducing kernel, then $k(\vx,\vx')-k(\vx)k(\vx')$ is also a reproducing kernel.

For now we take the claim as true.
Then notice that the AD kernel can also be written as
\[k_d((\vx,\vu),(\vx',\vu'))=
\begin{bmatrix}\vu^\top & 1\end{bmatrix}\tilde{\vD}(\vx,\vx')\begin{bmatrix}\vu^\top & 1\end{bmatrix}^\top +\begin{bmatrix}\vu^\top & 1\end{bmatrix}\tilde{\vA}(\vx,\vx')\begin{bmatrix}\vu^\top & 1\end{bmatrix}^\top, \]
for $\tilde{\vD}(\vx,\vx')$ a diagonal matrix with $i^{th}$ entry as $k_i(\vx-\vx')-k_i(\vx)k_i(\vx')$, and $\tilde{\vA}(\vx,\vx')$ a matrix whose entry at $i,j$ is $k_i(\vx)k_j(\vx')$  for $i, j\in[m+1]$. 
Since sums of kernels are also kernels, 
it suffices to show that each term is a kernel.
If the claim above is true, then the first term is a special case of the ADP kernel, and is therefore a reproducing kernel by Lemma 3 of \citet{contro1}.
The second term can be directly written as an inner product
$\begin{bmatrix}\vu^\top & 1\end{bmatrix} \tilde{\vA}(\vx,\vx')\begin{bmatrix}\vu^\top & 1\end{bmatrix}^\top=\begin{bmatrix}\vu^\top & 1\end{bmatrix}\vphi(\vx)\vphi(\vx')^\top\begin{bmatrix}\vu^\top & 1\end{bmatrix}^\top$ where $\vphi(\vx)=[k_1(\vx)~\dots~k_{m+1}(\vx)]^\top$.
Therefore, it is a kernel.

It now remains to prove the claim.
We proceed by showing that $k(\vx,\vx')-k(\vx)k(\vx')$ can be written as the inner product of some explicit feature representation.
Let $k(\vx,\vx')=\langle \vvarphi(\vx) , \vvarphi(\vx')\rangle$ for some feature function mapping to an arbitrary real Hilbert space $\mathcal H$, $\vvarphi:\mathcal X\to\mathcal H$. 
This must exist since $k$ is a reproducing kernel~\citep{berlinet2011reproducing}.
Then
\begin{align*}
	k(\vx,\vx')-k(\vx)k(\vx') &= \langle \vvarphi(\vx) , \vvarphi(\vx')\rangle-\langle \vvarphi(\vx) , \vvarphi(0)\rangle\langle \vvarphi(\vx') , \vvarphi(0)\rangle= \langle \vvarphi(\vx) , \left(\mathcal I - \mathcal P\right) \vvarphi(\vx') \rangle
\end{align*}
where $\mathcal I$ is identity operator and $\mathcal P:\mathcal H\to\mathcal H$ is defined as the linear operator $\mathcal P \vv = \vvarphi(0)\langle \vvarphi(0), \vv\rangle $ for $\vv \in \mathcal H$. 
We now argue that $\mathcal I - \mathcal P$ is a bounded self-adjoint positive operator; it is bounded, i.e. maps bounded subsets to bounded subsets, since $\mathcal I$ and $\mathcal P$(by Cauchy Schwarz) are bounded; it is positive, i.e., the quadratic form $\vv \mapsto \langle \vv , \left(\mathcal I - \mathcal P\right) \vv \rangle$ is positive semi-definite:
\[\langle \vv , \left(\mathcal I - \mathcal P\right) \vv \rangle = \|\vv\|^2_{\mathcal H} - \langle \vvarphi(0), \vv\rangle ^2 \geq  \|\vv\|^2_{\mathcal H} - (\|\vvarphi(0)\|_{\mathcal H}\|\vv\|_{\mathcal H})^2 =0. \]
The first inequality holds by Cauchy Schwarz, and the final equality holds because the kernel is normalized, i.e., $k(0)=\|\vvarphi(0)\|_{\mathcal H}^2=1$. 
It is self-adjoint: 
\begin{align*}
	\langle \vv , \left(\mathcal I - \mathcal P\right) \vw \rangle &= \langle \vv, \vw \rangle  -   \langle \vv,\vvarphi(0) \rangle \langle  \vvarphi(0), \vw \rangle, \\
	&= \langle \mathcal{I}\vv, \vw \rangle - \langle \vvarphi(0)\langle \vvarphi(0),\vv \rangle , \vw \rangle
	=\langle \left(\mathcal I - \mathcal P\right) \vv , \vw \rangle.
\end{align*}
Therefore by theorem 13.15 in \citet{functionalanalysis}, there exists a unique positive linear operator $\mathcal L$ such that
$\mathcal I-\mathcal P = \mathcal L^2$, and therefore
\[k(\vx,\vx')-k(\vx)k(\vx')=\langle \vvarphi(\vx) , \left(\mathcal I - \mathcal P\right) \vvarphi(\vx') \rangle = \langle \mathcal L\vvarphi(\vx) , \mathcal L\vvarphi(\vx') \rangle.\]
Therefore, we have constructed an explicit feature representation which proves the claim.

\subsection{Proof of Theorem~\ref{thrm:Affine-dense rf}}
First note that by assumption, the expectation of $\vpsi_i(\vx)^\top \vpsi_i(\vx')$ is equal 
$k_i(\vx-\vx')$. Thus it remains to show that $\vpsi_i(\vx)^\top \vpsi_j(\vx')$ estimates $k_i(\vx)k_j(\vx')$ when $i \neq j$. Recall that  by assumption, $\vpsi_i$ is constructed using i.i.d. samples from the inverse Fourier transform of the kernel $k_i$, i.e., $p_i(\vvartheta)$. Define $\zeta_{\vvartheta}(\vx)=e^{i\vvartheta^\top \vx}$. Then the expectation is given by,
\begin{align*}
	\mathbb{E}_{\vvartheta,\vvartheta'}[\zeta_{\vvartheta}(\vx)\overline{\zeta_{\vvartheta'}(\vx')}] &= \int_{\vvartheta,\vvartheta'\in\mathbb{R}^n} p_i(\vvartheta)p_j(\vvartheta')e^{i\vvartheta^\top \vx}e^{-i\vvartheta'^\top \vx'}d\vvartheta d\vvartheta', \\
	&=\int_{\vvartheta\in\mathbb{R}^n}p_i(\vvartheta)e^{i\vvartheta^\top \vx}d\vvartheta 
	\int_{\vvartheta'\in\mathbb{R}^n}p_j(\vvartheta')e^{-i\vvartheta'^\top \vx'}d\vvartheta', \nonumber\\
	&=\mathbb{E}_{\vvartheta}[\zeta_{\vvartheta}(\vx)]\mathbb{E}_{\vvartheta'}[\overline{\zeta_{\vvartheta'}(\vx')}]= k_i(\vx)k_j(\vx'). 
\end{align*}

\subsection{Proof of Proposition~\ref{prop:kernelapprox}}
The proof of this proposition is given in the first two steps of the proof of Proposition~\ref{proposition_one}.
\section{Gaussian Process Regression}\label{sec:gp-appendix}

An important advantage of kernel methods is that they are
amenable to theoretical guarantees and uncertainty characterization. 
Gaussian Process (GP) regression~\citep{gpbook} takes a Bayesian perspective,
and provides posterior mean and variance estimates on function values. 
Confidence intervals of this form can be derived for kernel predictions even in frequentist settings~\citep{GPtheorem}. 
Prior works have developed robust CCF-based controllers for unknown models by incorporating 
Gaussian process (GP) regression~\citep{contro1,adp}. 
We thus describe how to use our approximation methods for GP regression.
These results may be broadly of interest for any controller which makes use of GPs~\citep{ellipsoid, Learningbased_nonlinearMPC, NMPC,cautiousMPC,stochasticMPC}.

\subsection{Background}\label{sec:gp_background}
This section presents an overview of regression methods which explicitly model the uncertainty.
Bayesian linear regression is a probabilistic approach to regression analysis that models the relationship between a set of input vectors $\{\vs_i\}_{i=1}^N \in \mathcal{S} \subseteq \R^d$ and target output variables $\{z_i\}_{i=1}^N\in \mathcal{Z} \subseteq \R$ as
\begin{align*}
	z_i=h(\vs_i)+\epsilon_i, \quad h(\vs_i)=\vs_i^\top \vw,
\end{align*}  
where $\vw \in \R^d$ represents the linear model and $\{\epsilon_i\}_{i=1}^N \distas \mathcal{N}(0,\sigma^2_N)$ are the noise. In contrast to classical linear regression, Bayesian linear regression~(BLR) not only estimates the parameters of a linear model, but also provides a probabilistic interpretation of the model's uncertainty. More specifically, BLR treats the regression coefficients as random variables with a prior distribution, and computes the posterior distribution over these coefficients given $N$ input-output data pairs $\{(\vs_i, z_i)\}_{i=1}^N$.

The Bayesian linear model suffers from limited expressiveness. A simple approach to overcome this problem is to map the input vectors $\{\vs_i\}_{i=1}^N$ to a higher-dimensional feature space, where a linear relationship can be established more easily. If such a map exists, we call it a basis function.
Specifically, let $\vphi:\R^d\to\R^D$ maps an input vector $\vs \in \R^d$ to a feature vector $\vphi(\vs) \in \R^D$. Now the model becomes $h(\vs) = \vphi(\vs)^\top \vw$. Further, let the matrix $\Phi\in\R^{N\times D}$ and the vector $\vz \in \R^N$ be the aggregation of rows $\{\vphi(\vs_i)^\top\}_{i=1}^N$ and $\{z_i\}_{i=1}^N$, respectively. Assuming a Gaussian prior on the model weights, i.e., $\vw \sim \mathcal{N}(0,\vSigma_w)$, the posterior distribution of $h(\vs_*)$ at a query point $\vs_*$ is $h(\vs_*) \sim \mathcal{N}(\mu_*, \sigma_*^2)$ with,
\begin{equation}
	\begin{aligned}\label{posterier} 
		& \mu_* = \vphi_*^\top\vSigma_w\Phi^\top(\Phi\vSigma_w\Phi^\top+\sigma^2_N \vI_N)^{-1}\vz, \\
		& \sigma_*^2 = \vphi_*^\top\vSigma_w\vphi_*-\vphi_*^\top\vSigma_w\Phi^\top(\Phi\vSigma_w\Phi^\top+\sigma^2_N \vI_N)^{-1}\Phi\vSigma_w\vphi_*,
	\end{aligned}
\end{equation}
where we used the shorthand $\vphi_* = \vphi(\vs_*), \mu_* = \mu(\vs_*)$ and $\sigma_* = \sigma(\vs_*)$~\citep{gpbook}. 
This distributional perspective characterizes the uncertainty in the model prediction. 
One simple way to express the uncertainty is through a confidence interval. 
For Gaussian distributions, confidence intervals take the form,
\begin{align*}
	| \mu(\vs_*) - h(\vs_*) | \leq \beta \sigma(\vs_*),
\end{align*}
where $\beta\geq 0$ depends on the level of confidence.
Below in Theorem~\ref{thrm:GPtheorem}, we present a formal version of this confidence bound that holds even in the frequentist setting, meaning that it does not depend on the distributional assumptions or Bayesian priors.
However, the validity of the confidence interval does depend on the choice of basis functions, as this determines the complexity and richness of the modelled relationship between inputs $\vs$ and outputs $z$.
For data coming from a highly structured process, it may be reasonable to specify a basis of small dimension.
However, in general a suitable compact basis may not be known a priori.

Kernel methods are used to allow for expressive basis functions of arbitrarily high or infinite dimension.
Using the \emph{kernel trick}~\citep{scholkopf2018learning,muller2018introduction}, the posterior \eqref{posterier} can be computed using only inner products of basis functions.
A kernel function $k:\R^d\times \R^d\to \R$  generalizes the idea of inner products between basis functions. 
Using kernel functions in this context leads to the familiar Gaussian Process (GP) regression~\citep{gpbook}. 
GPs are commonly used as a nonparametric approach for representing complex functions.
GP regression corresponds to Bayesian regression in Reproducing Kernel Hilbert Spaces (RKHS), a specific class of function space denoted as $\mathcal{H}_k(\mathcal{S})$~\citep{wendland2004scattered}. The RKHS is equipped with the norm $\|h(\vs)\|_k := \sqrt{\li h(\vs), h(\vs) \ri}$ and is dense in the set of continuous functions, meaning that any continuous function can be represented arbitrarily well by kernel regression.

In the RKHS setting, \cite{GPtheorem} establishes the following frequentist confidence interval.
\begin{theorem}[\cite{GPtheorem}] \label{thrm:GPtheorem}
	Assume that the noise sequence $\{\epsilon_i\}_{i=1}^{\infty}$ is zero mean and uniformly bounded by $\sigma_N$. Let the target function $h: \mathcal{S} \to \R$ be a member of $\mathcal{H}_k(\mathcal{S})$ associated with a bounded kernel $k$, with its RKHS norm bounded by $B$. Then, with probability at least $1-\delta$, the following holds for all $\vs \in \mathcal{S}$ and $N \geq 1$:
	\begin{align*}
		| \mu(\vs_*) - h(\vs_*) | \leq \sqrt{2B^2 + 300 \gamma_{N+1} + \ln^3(\frac{N+1}{\delta})} \sigma(\vs_*),
	\end{align*}
	where $\gamma_{N+1}$ is the maximum information gain after getting $N+1$ data points.
\end{theorem}

The drawback of kernel methods is computation. Algorithms for fitting functions in an RKHS to data have superlinear complexity in the number of data points. In particular, computing the kernel approximator can be prohibitively expensive for large datasets. Solving \eqref{posterier} generally requires $O(N^3)$ time and $O(N^2)$ memory.

\subsection{Affine posterior}\label{affinenessinu}
In this section, we 
explicitly show the affineness of predicted $\mu$ and $\sigma$ in $\vu$ as products of input-dependent and independent parts. We will use these calculation in the case study \ref{sec:BLR-SOCP} to control an acrobat using CCF functions. Throughout, we define $\vy := [\vu^\top~~1]^\top$.

\paragraph{Kernel methods}
Under BLR assumptions, given a set of finite measurements of features and labels of the form $\{(\vs_i,z_i)\}_{i=1}^{N}$, where $z_i = h(\vs_i)+\epsilon_i$, and $\epsilon_i \sim \mathcal{N}(0, \lambda_{N}^2)$, a posterior distribution of $h(\vs)$ at a query point $\vs := (\vx, \vu)$ can be derived as follows: $h(\vx, \vu) \sim \mathcal{N}(\mu_{x}(\vu),\sigma_{x}(\vu)^{2})$ with 
\begin{align}\label{train_kernel}
	\mu_{x}(\vu) &:= \mu(\vx, \vu) = {\vz}^\top (\vK + \lambda_{N}^2 \vI_N )^{-1} \vk_{(x,u)}, \\
	\sigma_{x}(\vu)^{2} &:= \sigma(\vx, \vu)^2 = k\left((\vx,\vu), (\vx,\vu)\right)-\vk_{(x,u)} ^\top (\vK + \lambda_{N}^2 \vI_N )^{-1} \vk_{(x,u)},
\end{align}
where $\vK \in \mathbb{R}^{N\times N}$ is the Gram matrix whose entry at $i,j$ is given by $ [\vK]_{i,j} = k((\vx_i,\vu_i),(\vx_j,\vu_j))$ for $i,j \in [N]$. Further, $\vk_{(x,u)}=[k((\vx,\vu), (\vx_1,\vu_1))~~ \cdots~~k((\vx,\vu), (\vx_N,\vu_N))]^\top\in\mathbb{R}^N$, and ${\vz}\in \mathbb{R}^N$ is the vector containing the output measurements $z_i=h(s_i)+\epsilon_i$ for $i \in [N]$.

In the following, we will show the affineness of $\mu_{x}(\vu)$ and $\sigma_{x}(\vu)^{2}$, when $k((\vx,\vu),(\vx',\vu'))$ is an AD kernel~(Definition~\ref{def:AD_kernel}). We refer to \citet{adp} section 5 for the case of ADP kernel~(Definition~\ref{def:ADP_kernel}). 
To proceed let $\mathcal{X} = \{\vx_1,\dots,\vx_N\}$ and $\mathcal{Y}= \{\vy_1,\dots,\vy_N\}$ be the two sets containing the training data, where $\vy_i = [\vu^\top~~1]^\top$. Let $\vA(\vx,\vx')$ and $\vD(\vx,\vx')$ be as in Definition~\ref{def:AD_kernel}, and set $\vM(\vx,\vx') := \vD(\vx,\vx') + \vA(\vx,\vx')$. Further, define 
\[\vk_{train} = \mathrm{blkdiag}(\vy_1^\top,\dots,\vy_N^\top)[\vM(\vx_1,\vx)^\top~~\dots~~ \vM(\vx_N,\vx)^\top]^\top.\]
It's immediate that $\vk_{(x,u)}= \vk_{train} \vy$. Therefore,
\begin{align*}
	\mu_x(\vu) &= \underbrace{{\vz}^\top (\vK + \lambda_{N}^2 \vI_N )^{-1} \vk_{train}}_{=:\vXi_{ADK}(\mathcal{X},\mathcal{Y})} \vy, \\
	&= \vXi_{ADK}(\mathcal{X},\mathcal{Y}) \begin{bmatrix} \vu \\ 1\end{bmatrix}, \\
	\sigma_x(\vu)^{2} &= \vy^\top \vM(\vx,\vx)\vy - \vy^\top \vk_{train}^\top (\vK + \lambda_{N}^2 \vI_N )^{-1} \vk_{train} \vy,\\
	&= \tr{\vy} \underbrace{\big[\vM(\vx,\vx)-\vk_{train}^\top (\vK + \lambda_{N}^2 \vI_N )^{-1} \vk_{train}\big]}_{=: \vG_{ADK}(\mathcal{X},\mathcal{Y})} \vy, \\
	&= [\tr{\vu}~~1] \vG_{ADK}(\mathcal{X},\mathcal{Y}) \begin{bmatrix} \vu \\ 1\end{bmatrix}.
\end{align*}
Since every nonzero real vector can be scaled to have $1$ as the last entry, and from above $\sigma_x(\vu)^{2}=[\tr{\vu}~~1] \vG_{ADK} (\mathcal{X},\mathcal{Y})\begin{bmatrix} \vu \\ 1\end{bmatrix}$ for every $\vu$, $\vG_{ADK}(\mathcal{X},\mathcal{Y})$ is positive semi-definite.
\begin{align*}
	\implies \sigma_x(\vu)&= \norm{ \vOmega_{ADK}(\mathcal{X},\mathcal{Y})\begin{bmatrix} \vu \\ 1 \end{bmatrix}}_2.
\end{align*}

\paragraph{Random basis methods} Recall that using random features, the posterior mean and covariance can be approximated by
\begin{align}\label{train_rf}
	\hat{\mu}_x(\vu)= \vvarphi(\vx,\vu)^\top (\Phi^\top \Phi+\lambda_N \vI_D)^{-1}\Phi^\top \vz, \\ 
	\hat{\sigma}_x(\vu)^2=\lambda_N\vvarphi(\vx,\vu)^\top(\Phi^\top \Phi+\lambda_N \vI_D)^{-1 }\vvarphi(\vx,\vu).
\end{align}
ADP random features: From Definition \ref{def:adpb}, we know $\vvarphi(\vx,\vu) = \mathrm{blkdiag}(\vpsi_1(\vx),\dots,\vpsi_{m+1}(\vx))\vy$. Define $\vPsi_{ADP}(\vx):=\mathrm{blkdiag}(\vpsi_1(\vx),...,\vpsi_{m+1}(\vx))$. As a result
\begin{align*}
	\hat{\mu}_x(\vu)&=\tr{\vy} \underbrace{\tr{\vPsi(\vx)} (\Phi^\top \Phi+\lambda_N \vI_D)^{-1}\Phi^\top \vz}_{\vXi_{ADRF}(\mathcal{X},\mathcal{Y})}, \\
	&= [\tr{\vu}~~1] \vXi_{ADRF}(\mathcal{X},\mathcal{Y}),\\
	\hat{\sigma}_x(\vu)^2&= \tr{\vy} \underbrace{\lambda_N\tr{\vPsi(\vx)} (\Phi^\top \Phi+\lambda_N \vI_D)^{-1} \vPsi(\vx)}_{\vG_{ADPRF}(\mathcal{X},\mathcal{Y})} \vy,\\
	&=  [\tr{\vu}~~1] \lambda_N\vG_{ADPRF}(\mathcal{X},\mathcal{Y}) \begin{bmatrix} \vu \\ 1\end{bmatrix},\\
	\implies \hat{\sigma}_x(\vu)&= \norm{ \vOmega_{ADPRF}(\mathcal{X},\mathcal{Y})\begin{bmatrix} \vu \\ 1\end{bmatrix}}_2.
\end{align*}
\\
AD random features: From Definition~\ref{def:adb}, we know that $\vvarphi(\vx,\vy)=[\vpsi_1(\vx)~~\dots~~\vpsi_{m+1}(\vx)]\vy$. Define $\vPsi_{AD}(\vx):=[\vpsi_1(\vx)~~\dots~~\vpsi_{m+1}(\vx)]$. Then, similar to ADP random features, we have,
\begin{align*}
	\hat{\mu}_x(\vu)&= [\tr{\vu}~~1] \vXi_{ADRF}(\mathcal{X},\mathcal{Y}),\\
	\hat{\sigma}_x(\vu)^2&= [\tr{\vu}~~1]\lambda_N  \vG_{ADRF}(\mathcal{X},\mathcal{Y}) \begin{bmatrix} \vu \\ 1 \end{bmatrix},\\
	\implies \hat{\sigma}_x(\vu)&= \norm{ \vOmega_{ADRF}(\mathcal{X},\mathcal{Y})\begin{bmatrix} \vu \\ 1 \end{bmatrix}}_2.
\end{align*}

\subsection{Error bounds}\label{sec:errorbounds}

For the purposes of robust control, 
it is necessary to track how the approximation error accumulates in our computation of the posterior.

\cite{analysis} shows that with probability $1-\delta$, $\underset{\vx \in \mathcal X}{\sup}|\vpsi(\vx)^\top\vpsi(\vx)-k(\vx)|\leq \epsilon$ for $D\geq \frac{8(d+2\alpha_\epsilon)}{\epsilon^2}\left[\frac{2}{1+\frac{2}{d}}\log \frac{\sigma_p l}{\epsilon}+log\frac{\beta_d}{\delta}\right],$
where $l$ is the diameter of $\mathcal X$, $\sigma_p^2=\mathbb{E}_p\|\vomega\|^2$ and $\beta_d, \alpha_e$ are defined in Proposition 1 of \cite{analysis}.

Define $\vk_s \in \R^N$ to be a vector containing the kernel $k(\vs_i,\vs)$ for $i= 1,\dots,N$. 
Let $\vU \in \R^{N \times m}$ be a matrix with rows $\{\vu_i^\top\}_{i=1}^N$. We get the following error bound. 
\begin{proposition} \label{proposition_one}
	Assume each $i$-th element of $\vvarphi_C$ \eqref{Cdot-expand} is a member of $\mathcal{H}_{k_i}$ with bounded RKHS norm, for $i=1,\dots,m+1$. 
	Assume either the AD or ADP kernel with bounded kernels $k_i$ and that we have access to measurements $\vz$. 
	Assume $\|\vk_s\| \leq \sqrt{N} \kappa$ and that $\lambda_n=n\lambda$. 
	Let $\sigma_{max}$ be the max singular value of $\vU$. 
	Then with a probability of $1-(\delta_1+\delta_2)$ we have:
	\begin{equation}
		|\dot{C}(\vx,\vu)-\hat{\mu}_{\vx}(\vu)|
		\leq \beta\hat{\sigma}_{\vx}(\vu)+\epsilon(\nu\|\vu_x\|+\iota\\v|u_x\|^2+\Delta) \nonumber
	\end{equation}
	where $
	\nu := \frac{ \sigma_{max}}{\sqrt{N}\lambda}(\sigma_n + \frac{2\beta\kappa}{\sqrt{N}}+2\beta\epsilon), \iota=\frac{\beta\epsilon\sigma_{max}^2}{N\lambda},\\
	\Delta=\beta\Delta_\sigma+(\beta\kappa+\sqrt{N}\sigma_n)\Delta_\mu,
	\Delta_\mu=\frac{1}{\lambda\sqrt{N}}[1+\frac{\kappa\sigma_{max}}{N\sqrt{N}\lambda}+\frac{\kappa}{\sqrt{N}\lambda}],
	\Delta_\sigma=1+\epsilon+\frac{\kappa}{\sqrt{N}\lambda}.
	$
\end{proposition}

\begin{proof}
	we split the proof to several steps:
	\begin{enumerate}
		
		\item Approximating AD kernel pointwise:
		\begin{align}
			|k_i(\vx)k_i(\vx')-\hat{k}_i(\vx)\hat{k}_i(\vx')|
			&= |k_i(\vx)(k_i(\vx')-\hat{k}_i(\vx'))+(k_i(\vx)-\hat{k}_i(\vx))\hat{k}_i(\vx)|, \nonumber\\
			&\leq \epsilon \: \textit{max}(|k_i(\vx)|,|\hat{k}_i(\vx')|), \nonumber\\
			&\leq \epsilon, \label{ad-pointwise}
		\end{align}
		where we used the assumption that $|k_i(\vx)|\leq 1$ for all $\vx \in \mathcal{X}$ and all $i \in [m+1]$. Let $\vE:=\vA+\vD=[e_{ij}]_{\{i,j\}}$, that is, $e_{ij}$ is the element in $i^{th}$ row and $j^{th}$ column of $\vE$, where $\vA,\vD$ are as in Definition~\ref{def:AD_kernel}. Using~\eqref{ad-pointwise}, we conclude that $e_{ij}-\hat{e}_{ij}\leq\epsilon$ for all $1\leq i,j \leq {m+1}$. This further implies,
		\begin{align*}
			|k_d((\vx,\vu),(\vx',\vu'))-\hat{k}_d((\vx,\vu),(\vx',\vu'))|&=|\sum_{1\leq i,j \leq m+1}y_i(e_{ij}-\hat{e}_{ij})y_j'|\\
			&\leq \epsilon(\vu^\top \vu'+1), 
		\end{align*}
		where $k_d((\vx,\vu),(\vx',\vu'))$ is in Definition~\ref{def:AD_kernel} and $y_i, y_j'$ denote the $i^{th}$ and $j^{th}$ elements of $\vy := [\vu^\top~~1]^\top$ and $\vy':= [\vu'^\top~~1]^\top$, respectively.
		\item Approximating ADP Kernel pointwise:
		\begin{align*}
			|k_c((\vx,\vu),(\vx',\vu'))-\hat{k}_c((\vx,\vu),(\vx',\vu'))|&=|\sum_{1\leq i \leq m+1}y_i(e_{ii}-\hat{e}_{ii})y_i'|\\
			&\leq \epsilon(\vu^\top \vu'+1), 
		\end{align*}
		
		\item Approximating ADP, AD kernel matrices:\\
		since we have the same bound on the estimation error of ADP kernel and AD kernel, the following proofs hold for both kernels:
		\begin{align*}
			\|\vK-\hat{\vK}\|_2 &\leq \epsilon \|[\vu_i^\top \vu_j+1]_{i,j}\|_2 \leq \epsilon \sigma_{max}^2 + \epsilon (m+1)\\
			\|\vk_s-\hat{\vk}_s\| &\leq \epsilon \|[\vu_x^\top  \vu_i+1]_i\|
			\leq \epsilon \|\vu_x\|.\|[\vu_1,...,\vu_N]\| +\epsilon\sqrt{N} \leq \epsilon \sigma_{max}\|\vu_x\|+ \epsilon\sqrt{N}
		\end{align*}

		\item Approximating the mean:
		\begin{align*}
			|\mu_x(\vu)-\hat{\mu}_x(\vu)| &=
			\|\vz^\top\|\|(\hat{\vK}+\lambda_n\vI)^{-1}(\vk_s-\hat{\vk}_s)+((\vK+\lambda_n\vI)^{-1}-(\hat{\vK}+\lambda_n\vI)^{-1})\vk_s\|\\
			&\leq \frac{\|\vz\|}{\lambda_n}\|\hat{\vk}_s-\vk_s\|+\frac{\|\vz\|. \|\hat{\vK}-\vK\|}{\lambda_n^2}\|\vk_s\|\\
			& \leq \frac{\sqrt{N}\sigma_n}{N\lambda}\|\hat{\vk}_s-\vk_s\|+\frac{\sqrt{N}\sigma_n\kappa}{n^2\lambda^2}\|\hat{\vK}-\vK\|\\
			&\leq \frac{\sqrt{N}\sigma_n}{N\lambda}(\epsilon \sigma_{max}\|\vu_x\|+ \epsilon\sqrt{N})+ \frac{\sqrt{N}\sigma_n\kappa}{N^2\lambda^2} (\epsilon \sigma_{max}^2 + \epsilon N) \quad\\
			&\leq \epsilon \|\vu_x\| \frac{\sigma_{max}\sigma_n}{\sqrt{N}\lambda} + \sqrt{N}\sigma_n\Delta_\mu
		\end{align*}
		Where we used $(\vK+\lambda \vI)^{-1}-(\hat{\vK}+\lambda \vI)^{-1}=(\hat{\vK}+\lambda \vI)^{-1}(\hat{\vK}-\vK)(\vK+\lambda \vI)^{-1}$, and that the smallest eigenvalue of $\hat{\vK}+\lambda \vI$ and $\vK+\lambda \vI$ is at least $\lambda$.
		\item Approximating variance: Assume $\sigma_x(\vu) + \hat{\sigma} _{x}(\vu) \leq 1$
		\begin{align*}
			|\sigma_x(\vu) - \hat{\sigma}_x(\vu)| &\leq |\sigma_x(\vu)^2 - \hat{\sigma}_{x}(\vu)^2| \\
			&\leq \epsilon + |\vk_s[(\vK+\lambda)^{-1}\vk^\top_s-(\hat{\vK}+\lambda)^{-1}\hat{\vk}^\top_s]+[\vk_s-\hat{\vk}_s](\hat{\vK}+\lambda)^{-1}\hat{\vk}^\top_s|\\
			&\leq \epsilon + \|\vk_s\| \frac{|\mu_x(\vu)-\hat{\mu}_{x}(\vu)|}{\|\vz\|}+\frac{\epsilon \sigma_{max}\|\vu_x\|+ \epsilon\sqrt{N}}{N\lambda}\|\hat{\vk}_s\|\\
			&\leq \epsilon + \kappa (\epsilon \|\vu_x\| \frac{\sigma_{max}^2}{N\lambda} + \Delta_\mu) \\
			&\quad +\frac{\epsilon \sigma_{max}\|\vu_x\|+ \epsilon\sqrt{N}}{N\lambda}(\kappa+\epsilon \sigma_{max}^2\|\vu_x\|+ \epsilon\sqrt{N})\\
			&\leq \|\vu_x\|\frac{2\epsilon \sigma_{max}^2}{\sqrt{N}\lambda}(\frac{\kappa}{\sqrt{N}}+\epsilon) + \|\vu_x\|^2\frac{\epsilon^2 \sigma_{max}^2}{N\lambda}+\Delta_\sigma+\kappa\Delta_\mu
		\end{align*}
		\item Bounds on total error: with a probability of $1\!-\!(\delta_1\!+\!\delta_2)$:
		\begin{align*}
			|\dot{C}_x(\vu)-\hat{\mu}_x(\vu)|&\leq |\mu_x(\vu)-\dot{C}_x(\vu)| + |\mu_x(\vu)-\hat{\mu}_x(\vu)| \\
			&\leq \beta\hat{\sigma}_x(\vu)+ \beta|\sigma_x(\vu)-\hat{\sigma}_x(\vu)| + |\mu_x(\vu)-\hat{\mu}_x(\vu)|\\
			&\leq \beta\hat{\sigma}_x(\vu)+ \|\vu_x\|\nu + \|\vu_x\|^2\iota+\Delta_s
		\end{align*}
	\end{enumerate}
\end{proof}

\section{Empirical Compound Random Basis Comparison}\label{sec:rf_comparison}

In this Appendix, we perform experiments with synthetic data to demonstrate the relationship between training time and RMSE for ADP-RF and AD-RF models. 
Specifically, we use the following control-affine function to generate the data,
\begin{align}\label{u_analysis}
	h_m(\vx,\vu) = 3 \sin(2\pi \vx^\top \vw_1) - 2 \sin(4\pi \vx^\top\vw_2)+ \sum_{j=1}^m (\gamma_j \sin(2\pi \vx^\top \vw_{j+2}))u_j + \epsilon,
\end{align}
where  $\vw_1, \vw_2 \in \R^n$ parameterize $\vf(\vx) := 3 \sin(2\pi \vx^\top \vw_1) - 2 \sin(4\pi \vx^\top\vw_2)$, and $\{\vw_{j+2}\}_{j=1}^m \in \R^n$ parameterize $\vg(\vx) := [\gamma_1 \sin(2\pi \vx^\top \vw_{3}) ~ \gamma_2 \sin(2\pi \vx^\top \vw_{4}) ~ \cdots~ \gamma_m \sin(2\pi \vx^\top \vw_{m+2})]$, such that $h_m(\vx,\vu) = \vf(\vx) + \vg(\vx)\vu + \epsilon$ is affine in $\vu$. All the (entries of) weights $\vw$ and $\gamma$ are sampled uniformly at random from $[0,1)$.

We use the function \eqref{u_analysis} to generate data for varied input dimension $m=1,...,20$. For each $m=k$, to generate $h_m(\vx,\vu)$,  we use the previous weights $\vw_1,\dots,\vw_{k+1},\gamma_1,\dots,\gamma_{k-1}$ and only generate new weights for $j=k$, that is $\vw_{k+2}$ and $\gamma_k$. 

For each input dimension, we sample $1000$ values of $\vx_i\in\mathbb R^6$ and $\vu_i\in\mathbb R^m$ uniformly at random from $[0,1)$.
For each $i$, we set the label $y_{i}=h_m(\vx_i,\vu_i)+\epsilon_i$ where
$\epsilon_i \sim \mathcal{N}(0,0.01)$.
Using this dataset, we train and evaluate AD-RF, ADP-RF on a 90/10 train/test split.
For all methods, we use $\lambda=1$ and rbf $\gamma=1$. 
To choose feature dimensions, we vary the state-dependent basis dimension starting from $22$, and is incremented by $22$ for $10$ steps; this is roughly chosen based on the widely considered best random features dimension of $\sqrt{n}\log{n}$. The ADP compound basis dimension would be the multiplication of these dimensions by $m+1$. For fair comparison, we make sure that both the compound basis dimensions match.
At each dimension, we resample the random features 10 times, and record the training time and test RMSE.
\begin{figure}[t!]
	\begin{centering}
		\includegraphics[width=0.32\linewidth]{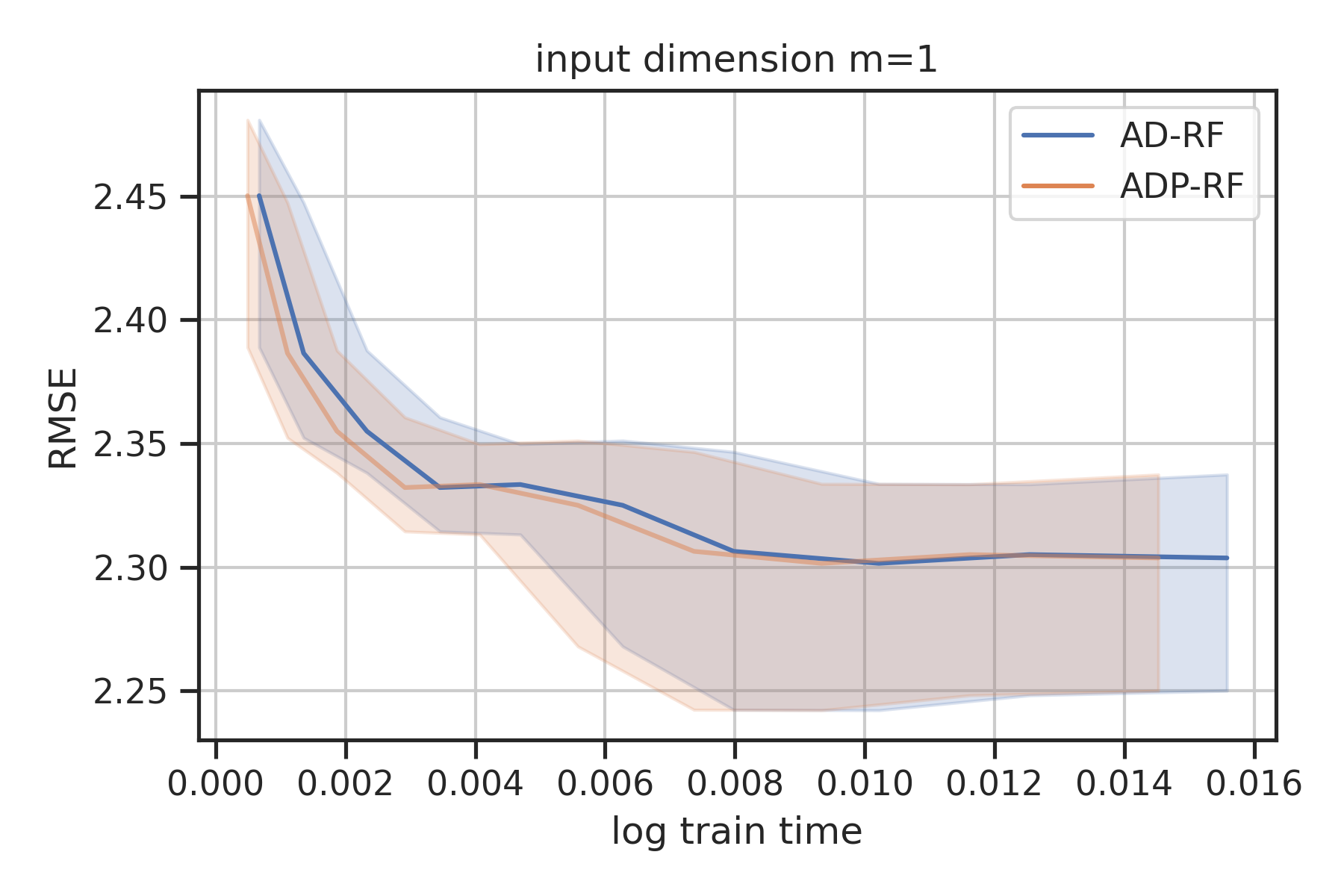} 
		~~ 
		\includegraphics[width=0.32\linewidth]{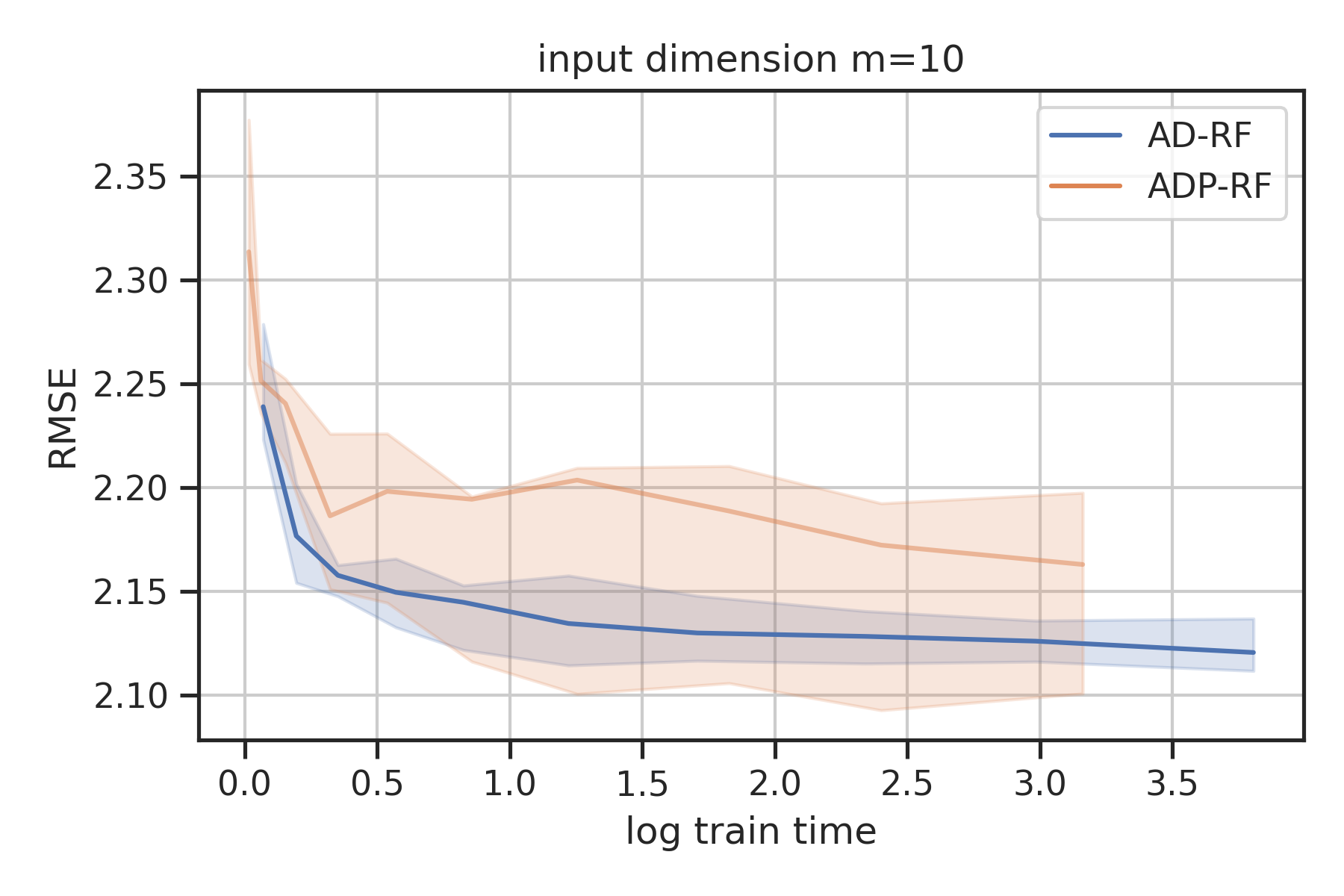}
		~~ 
		\includegraphics[width=0.32\linewidth]{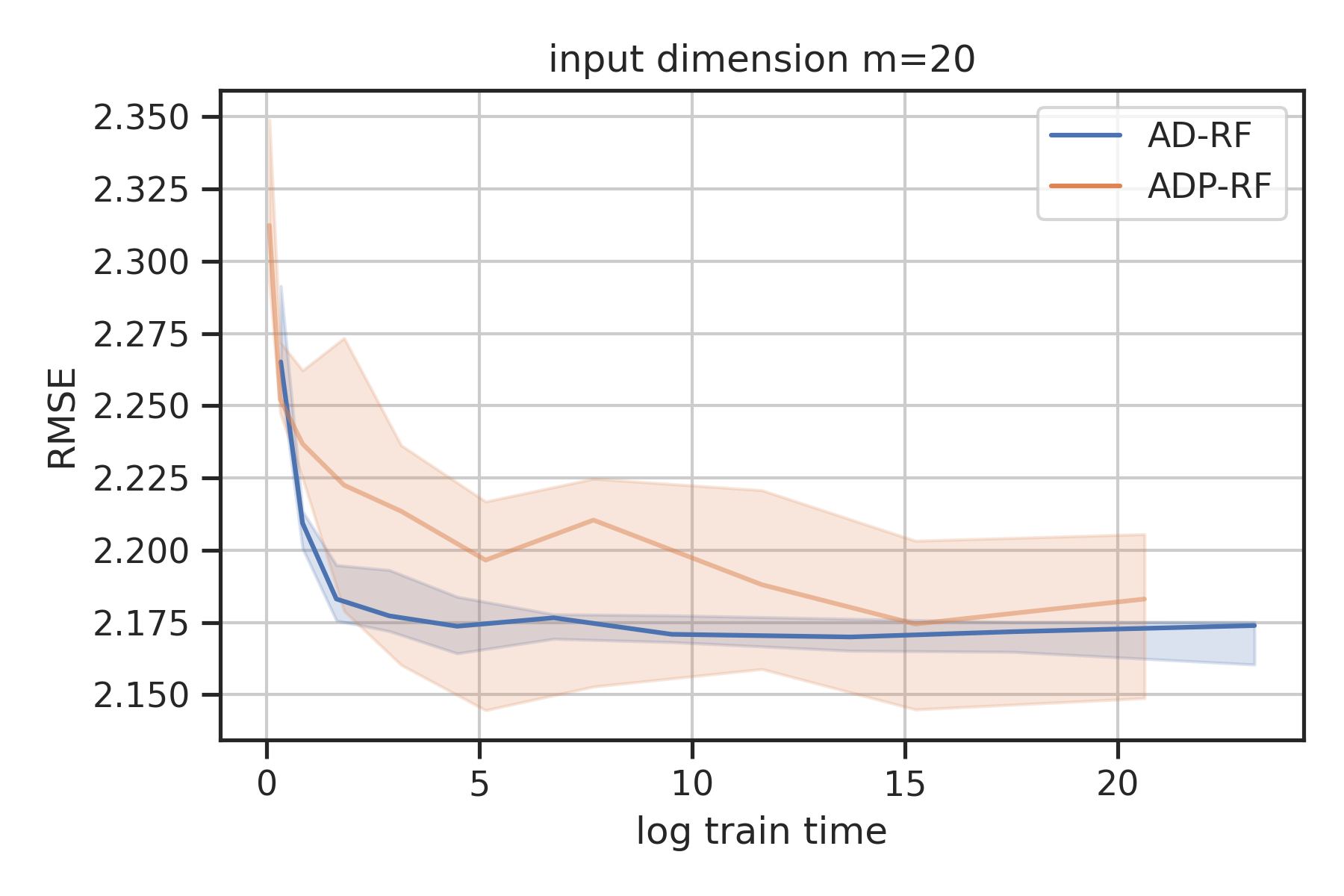}
	\end{centering}
	\caption{Evaluation of models, comparing prediction accuracy on 100 test data points against average training time on 900 data points for $m=1,10,20$ respectively. Random features are sampled 10 times at matching compound dimensions; the panels display median and quartile RMSE over the trials. Increasing $m$ demonstrates the advantage of AD-RF basis over ADP-RF in RMSE for fixed train-time.}\vspace{-20pt}
	
	\label{fig:prediction_appendix}
\end{figure} 

Figure \ref{fig:prediction_appendix} plots the median and quartile RMSE against the average training time. For lower dimensional inputs $\vu$, at a given train time, ADP-RF and AD-RF perform similarly. However with increasing $m$, the performance gap between AD-RF and ADP-RF becomes larger. Specifically at larger $m$'s, AD-RF converges faster at a lower RMSE; This means AD-RF reaches its optimal RMSE at a lower train-time, which implies that at lower feature dimensions, AD-RF captures the complexity of the model, whereas to ensure the same for ADP-RF, we need more complex and higher dimensional features.

\section{Robust CCF Control}\label{sec:robust_ccf}
In Section~\ref{sec:case_study_CFC}, we present a \emph{certainty-equivalent} approach to data-driven control using CCF control as a case study.
Here, we additionally present a \emph{robust} approach.
It allows for synthesizing a data-driven control law which robustly enforces the constraint in~\ref{eq:CCF-QP}.

\subsection{Robust Bayesian Data-driven Controller}\label{sec:BLR-SOCP}
In this section, we show how to construct a robust data driven control law given the control-affine basis functions (introduced in Section~\ref{sec:comparison}). Similar to Section~\ref{sec:case_study_CFC}, we suppose that the dynamics $\vf$ and $\vg$ are unknown, so the robust CCF controller cannot be directly implemented. We assume that a valid CCF $C$ and comparison function $\alpha$ for the unknown true system is given. Given a sampled trajectory $\{(\vx_i,\vu_i)\}_{i=1}^N$,
we construct a control affine modelling problem for $\dot C\!:\!\mathcal X\!\times\!\mathcal U\!\to\!\R$ as described in Example~\ref{example_03}. Specifically, we use GP regression to model the uncertainty  
and compute the mean $\mu_{\vx}(\vu):=\mu(\vs)$
and variance $\sigma_{\vx}(\vu):=\sigma(\vs)$  according to~\eqref{posterier}.
As first proposed in the GP context by \cite{contro1},
we make use of the $1-\delta$ confidence interval presented in Theorem~\ref{thrm:GPtheorem} 
to construct an optimization-based controller. 
The data-driven min-norm stabilizing feedback control law $\vu^* \colon \mathcal X \to \mathcal U$ is defined as
\begin{align*}\tag{\textbf{BLR-CCF-SOCP}} \label{BLR-CCF-SOCP}
	\vu^{*}(\vx) &= \underset{\vu\in \mathcal U}{\arg\min } \norm{\vu}_2^2\\
	&~\text{s.t.}\quad \mu_{\vx}(\vu) + \beta \sigma_{\vx}(\vu) + \alpha(C(\vx)) \leq 0
\end{align*}

This optimization problem will be a Second-Order Cone Program (SOCP) as long as the constraint is a conic in $\vu$.
This follows from the form of the mean and variance function,
and can be guaranteed as long as the basis function
$\vphi$ is \emph{affine} in $\vu$.
This is a natural requirement due to the affine structure of the dynamics.
\begin{equation}
	\begin{aligned}\label{Cdot-expand}
		\dot C(\vx, \vu)& = \nabla C(\vx)^\top \vf(\vx) + (\nabla C(\vx)^\top \vg(\vx))\vu, \\
		&= \vvarphi_C \begin{bmatrix}\vu \\ 1 \end{bmatrix}
	\end{aligned}
\end{equation}
where $\vvarphi_C \in \R^{1\times(m+1)}$.  

When a random features approximation is used, there is an additional error term that much be accounted for in proper robust control.
Leveraging the error analysis presented in Section~\ref{sec:errorbounds}, we present~\eqref{RF-CCF-SOCP} which is both computationally efficient and robust. 

\begin{align*}
	\tag{\textbf{RF-CCF-SOCP}} \label{RF-CCF-SOCP}
	\vu^*(\vx) &= \underset{\vu\in \mathbb{R}^m}{\arg\min } \|\vu\|_2^2 \\
	&\text{s.t.} \quad  \hat \mu_x(\vu) + \beta \hat \sigma_x(\vu) +\epsilon(\nu\|\vu_x\|+\iota\|\vu_x\|^2+\Delta) + \alpha(C(\vx))\leq 0
\end{align*}


\section{Experimental Details}
\subsection{Double pendulum dynamics derivation}\label{sec:dipderivation}
We consider a frictionless two-link pendulum with torque $\tau_1$ applied at a fixed base, where the first link is attached, and torque $\tau_2$ applied at the opposite end, where the second link is attached.  
The links are modeled as point masses $m_1$ and $m_2$ at lengths $l_1$ and $l_2$ from the joints. We define $\theta_1$ as the angle of the first link, measured from the upright positive, and $\theta_2$ as the angle of the second link, measured from the first link. The corresponding angular rates are $\dot{\theta}_1$ and $\dot{\theta}_2$. 

Letting $\vq = [\theta_1, \theta_2 ]$, 
the total kinetic energy of the system is 
given by
\begin{align}
	T(\vq, \dot{\vq}) &= \frac{1}{2}((m_1 + m_2)l_1^2 + 2 m_2 l_1 l_2 \cos{\theta_2} + m_2 l_2^2) \dot{\theta}_1^2+ (m_2 l_1 l_2 \cos{\theta_2} + m_2 l_2^2) \dot{\theta}_1 \dot{\theta}_2 \nonumber \\
	& \qquad+ \frac{1}{2}m_2 l_2^2 \dot{\theta}_2^2,
\end{align}
and the potential energy of the system is given by
\begin{align}
	U(\vq) =
	(m_1 + m_2) g l_1 \cos{\theta_1} + m_2 g l_2 \cos{(\theta_1 + \theta_2)}.
\end{align}
As a result, the Lagrangian of the system takes the following form,
\begin{align}
	L(\vq, \dot{\vq}) &= T(\vq, \dot{\vq}) -  U(\vq) ,\nonumber \\
	& = \frac{1}{2}((m_1 + m_2)l_1^2 + 2 m_2 l_1 l_2 \cos{\theta_2} + m_2 l_2^2) \dot{\theta}_1^2 + (m_2 l_1 l_2 \cos{\theta_2} + m_2 l_2^2) \dot{\theta}_1 \dot{\theta}_2, \nonumber \\
	& \qquad + \frac{1}{2}m_2 l_2^2 \dot{\theta}_2^2 - (m_1 + m_2) g l_1 \cos{\theta_1} - m_2 g l_2 \cos{(\theta_1 + \theta_2)}.
\end{align}
Now we can write the Lagrange equations as follows
\begin{align}
	\frac{d}{dt}\derp{L(\vq, \dot{\vq})}{\dot{\vq}} - \derp{L(\vq, \dot{\vq})}{\vq} = \boldsymbol{\tau}.
\end{align}
The partial derivatives are given by
\begin{align}
	\derp{L(\vq, \dot{\vq})}{\dot{\theta}_1} &= ((m_1 + m_2)l_1^2 + 2 m_2 l_1 l_2 \cos{\theta_2} + m_2 l_2^2) \dot{\theta}_1 + (m_2 l_1 l_2 \cos{\theta_2} + m_2 l_2^2) \dot{\theta}_2, \nonumber \\
	\derp{L(\vq, \dot{\vq})}{{\theta}_1} &= (m_1 + m_2) g l_1 \sin{\theta_1} + m_2 g l_2 \sin{(\theta_1 + \theta_2)}, \nonumber \\
	\derp{L(\vq, \dot{\vq})}{\dot{\theta}_2} &= (m_2 l_1 l_2 \cos{\theta_2} + m_2 l_2^2) \dot{\theta}_1 + m_2 l_2^2 \dot{\theta}_2, \nonumber \\
	\derp{L(\vq, \dot{\vq})}{{\theta}_2} &= -m_2 l_1 l_2 \sin{\theta_2} \dot{\theta}_1^2 - m_2 l_1 l_2 \sin{\theta_2}\dot{\theta}_1\dot{\theta}_2 + m_2 g l_2 \sin(\theta_1 + \theta_2). \nonumber
\end{align}
This further implies,
\begin{align}
	\frac{d}{dt}\derp{L(\vq, \dot{\vq})}{\dot{\theta}_1} &= ((m_1 + m_2)l_1^2 + 2 m_2 l_1 l_2 \cos{\theta_2} + m_2 l_2^2) \ddot{\theta}_1 + (m_2 l_1 l_2 \cos{\theta_2} + m_2 l_2^2) \ddot{\theta}_2 \nonumber \\
	& \qquad - 2 m_2 l_1 l_2 \sin{\theta_2} \dot{\theta}_1 \dot{\theta}_2 - m_2 l_1 l_2 \sin{\theta_2} \dot{\theta}_2^2, \nonumber \\
	\frac{d}{dt}\derp{L(\vq, \dot{\vq})}{\dot{\theta}_2} &= (m_2 l_1 l_2 \cos{\theta_2} + m_2 l_2^2) \ddot{\theta}_1 +  m_2 l_2^2 \ddot{\theta}_2 - m_2l_1l_2\sin{\theta_2}\dot{\theta}_1\dot{\theta}_2. \nonumber
\end{align}
With these results, we can write the following Lagrange equations
\begin{align}
	&\frac{d}{dt}\derp{L(\vq, \dot{\vq})}{\dot{\theta}_1} - \derp{L(\vq, \dot{\vq})}{{\theta}_1} = \tau_1, \nonumber \\
	\implies & ((m_1 + m_2)l_1^2 + 2 m_2 l_1 l_2 \cos{\theta_2} + m_2 l_2^2) \ddot{\theta}_1 + (m_2 l_1 l_2 \cos{\theta_2} + m_2 l_2^2) \ddot{\theta}_2 \nonumber \\
	& - 2 m_2 l_1 l_2 \sin{\theta_2} \dot{\theta}_1 \dot{\theta}_2 - m_2 l_1 l_2 \sin{\theta_2} \dot{\theta}_2^2 - (m_1 + m_2) g l_1 \sin{\theta_1} - m_2 g l_2 \sin{(\theta_1 + \theta_2)} = \tau_1. \nonumber
\end{align}
Similarly,
\begin{align}
	&\frac{d}{dt}\derp{L(\vq, \dot{\vq})}{\dot{\theta}_2} - \derp{L(\vq, \dot{\vq})}{{\theta}_2} = \tau_2, \nonumber \\
	\implies & (m_2 l_1 l_2 \cos{\theta_2} + m_2 l_2^2) \ddot{\theta}_1 +  m_2 l_2^2 \ddot{\theta}_2 - m_2l_1l_2\sin{\theta_2}\dot{\theta}_1\dot{\theta}_2 +m_2 l_1 l_2 \sin{\theta_2} \dot{\theta}_1^2 \nonumber \\
	&+ m_2 l_1 l_2 \sin{\theta_2}\dot{\theta}_1\dot{\theta}_2- m_2 g l_2 \sin(\theta_1 + \theta_2) = \tau_2, \nonumber \\
	\implies & (m_2 l_1 l_2 \cos{\theta_2} + m_2 l_2^2) \ddot{\theta}_1 +  m_2 l_2^2 \ddot{\theta}_2 +m_2 l_1 l_2 \sin{\theta_2} \dot{\theta}_1^2 - m_2 g l_2 \sin(\theta_1 + \theta_2) = \tau_2. \nonumber
\end{align}
From the above Lagrange equations, we write the following manipulator equation,
\begin{align}
	\vM(\vq) \ddot{\vq} + \vC(\vq, \dot{\vq}) \dot{\vq} = \boldsymbol{\tau}_g(\vq) + \vB \vu.
\end{align}
where
\begin{align}
	\vM(\vq) &:= \begin{bmatrix} (m_1 + m_2)l_1^2 + 2 m_2 l_1 l_2 \cos{\theta_2} + m_2 l_2^2 & m_2 l_1 l_2 \cos{\theta_2} + m_2 l_2^2 \\ m_2 l_1 l_2 \cos{\theta_2} + m_2 l_2^2 & m_2 l_2^2 \end{bmatrix}, \\
	\vC(\vq, \dot{\vq}) &:= \begin{bmatrix} - 2 m_2 l_1 l_2 \sin{\theta_2} \dot{\theta}_2 & - m_2 l_1 l_2 \sin{\theta_2} \dot{\theta}_2 \\ m_2 l_1 l_2 \sin{\theta_2} \dot{\theta}_1 & 0 \end{bmatrix}, \\
	\boldsymbol{\tau}_g(\vq) &:= \begin{bmatrix} (m_1 + m_2) g l_1 \sin{\theta_1} + m_2 g l_2 \sin{(\theta_1 + \theta_2)} \\  m_2 g l_2 \sin(\theta_1 + \theta_2) \end{bmatrix}, \\
	\vB &:= \begin{bmatrix} 1 \\ 1 \end{bmatrix}, \quad \vu := \begin{bmatrix} \tau_1 \\ \tau_2 \end{bmatrix}, \quad \vq := \begin{bmatrix} \theta_1 \\ \theta_2 \end{bmatrix}, \quad \dot{\vq} = \begin{bmatrix} \dot{\theta}_1 \\ \dot{\theta}_2\end{bmatrix},  \quad \ddot{\vq} = \begin{bmatrix} \ddot{\theta}_1 \\ \ddot{\theta}_2\end{bmatrix}.
\end{align}
Therefore, the state of the acrobat is $\vx=(\vq, \dot{\vq})=(\theta_1,\theta_2, \dot{\theta_1}, \dot{\theta_2})$, where, $\theta_1,\theta_2 \in [0,\pi]$ and $\dot{\theta_1}, \dot{\theta_2} \in \R$. The input $\vu=[\tau_1, \tau_2]$ is of 2 dimensions.
The manipulator equation can be written as the following control affine dynamics.
\begin{align}
	\dot{\vx}= \underbrace{\begin{bmatrix} \dot{\vq} \\ \vM(\vq)^{-1} \left(-\vC(\vq, \dot{\vq})\dot{\vq} + \boldsymbol{\tau}_g(\vq)\right) \end{bmatrix}}_{\vf(\vx)} + \underbrace{\begin{bmatrix} 0\\ \vM(\vq) ^{-1}\vB \end{bmatrix}}_{\vg(\vx)}\vu
\end{align}

\subsection{CCF Modelling Problem}\label{sec:ccf_modeling}
We consider a control affine modelling problem in the setting of learning the residual errors from a nominal dynamics model (Example~\ref{example_04})
for a CCF (Example~\ref{example_03}).
The true dynamics model $\vf,\vg$ is defined according to the equations above with $m_1=m_2=l_1=l_2=1$ while
the nominal dynamics model $\tilde\vf,\tilde\vg$ is defined with incorrect values of mass and length, $\tilde m_1=\tilde m_2=\tilde l_1=\tilde l_2=0.6$.
The CCF is a CLF and is defined to ensure stability to the origin:
\[C(\vx)=\vx^\top P \vx,\quad P=\begin{bmatrix}
	12 & 0& 3.16&0\\
	0&12&0&3.16\\
	3.16&0&4.04&0\\
	0&3.16&0&4.04
\end{bmatrix}.\]
Using the nominal model,
$\dot{\tilde C}(\vx,\vu) = \nabla C(\vx)^\top (\tilde \vf(\vx) + \tilde \vg(\vx)\vu)$.
The goal of the modelling problem is to learn the residual $ \dot C(\vx,\vu)- \dot{\tilde C}(\vx,\vu)$.
Given a sampled trajectory $\{\vx_i,\vu_i\}_{i=1}^{L+1}$, we compute $\{C(\vx_i)\}_{i=1}^{L+1}$ and use forward finite differencing to approximate $\{\hat{\dot C}_i\}_{i=1}^L$.
Then the regression targets are defined as $z_i = \hat{\dot C}_i - \dot{\tilde C}(\vx_i,\vu_i)$.

\subsection{Nominal Control Data Collection}
\label{sec:data_collection}
We collect data using a controller designed with the nominal dynamics models.
The control law is given by~\ref{eq:CCF-QP} with $c_1=25$, $\vu_d(\vx)$ a feedback linearizing controller designed for nominal dynamics $\tilde \vf,\tilde\vg$, $C(\vx)=\vx^\top P\vx$ defined above, $\tilde\vf,\tilde \vg$ used in place of $\vf,\vg$, and $\alpha(c)=0.725c$.
Additionally, the hard constraint is replaced with a slack variable with penalty coefficient $1e6$.

The nominal control law is simulated in closed-loop with the true dynamics using Runge-Kutta 4(5) at 10 Hz.
We collect $E=226$ trajectories starting from different initial conditions.
The initial conditions comprise of a meshgrid of coordinates of the different initial states.
Each trajectory is $5$ seconds long, resulting in $L+1=50$ sampled points, the final datasize is of size $11074$.

For the prediction experiments, the data is split into a test and train set shuffled at random.

\subsection{Closed-Loop Experiments}\label{sec:closed_loop}

We evaluate six controllers starting from initial state $\vx_0=[2,0,0,0]$.
All simulations during data collection and  evaluation use Runge-Kutta 4(5) at 10 Hz.

The nominal controller is described above. 
The oracle controller is given by~\ref{eq:CCF-QP} with the same parameters as the nominal, except that the constraint uses the true dynamics $\vf,\vg$.
Each of the four data-driven controllers
is defined using an affine model of the residual $\hat h$.
The control law is defined by~\ref{eq:CCF-QP} where $\dot C(\vx,\vu)$ is replaced with  $\dot{\tilde C}(\vx,\vu)+\hat h(\vx,\vu)$,
the slack penalty is $1e6\cdot (t+1)$ where $t$ is the time in seconds,
and otherwise the parameters are the same as for the nominal/oracle controllers.
Section~\ref{affinenessinu} presents the precise affine form in terms of the training data.

Each data-driven model $\hat h$ is trained in an episodic manner.
The process is warm started with subsampling the nominal grid data at a rate of 1/5 resulting in 2215 data points.
This initial training dataset defines an affine model (AD-K, ADP-K, AD-RF, or ADP-RF) which in turn defines a data-driven controller.
We simulate the data-driven controller in closed loop starting  from $\vx_0$ for 10 seconds, add the resulting data to the training set, and retrain.
We repeat for 10 episodes, resulting in a training set size of 3215, and report the performance of the final controller.

\end{document}